\documentclass{article}

\usepackage{style/iclr2027_conference}

\usepackage{times}
\usepackage{natbib}
\usepackage{hyperref}
\usepackage{etoc}
\usepackage{booktabs}
\usepackage{amsmath}
\usepackage{amssymb}
\usepackage{amsthm}
\newtheorem{proposition}{Proposition}
\newtheorem{assumption}{Assumption}
\newtheorem{corollary}{Corollary}
\newcommand{\defeq}{\mathrel{:=}}
\usepackage{graphicx}
\usepackage{tabularx}
\usepackage{xcolor}
\usepackage{listings}
\usepackage{float}
\usepackage{placeins}
\usepackage{pdflscape}
\usepackage{array}
\usepackage{enumitem}

\lstdefinestyle{caseStudyTranscript}{
  basicstyle=\ttfamily\scriptsize,
  breaklines=true,
  breakatwhitespace=false,
  columns=fullflexible,
  keepspaces=true,
  showstringspaces=false,
  frame=single,
  xleftmargin=0.5em,
  xrightmargin=0.5em,
  literate={→}{{$\rightarrow$}}1
}

\definecolor{casepromptblue}{RGB}{31,86,140}
\definecolor{caseanswergreen}{RGB}{20,110,70}
\definecolor{casemetagray}{RGB}{120,120,120}
\definecolor{caseboxbg}{RGB}{247,248,250}
\definecolor{caseboxrule}{RGB}{200,205,212}
\definecolor{casetitlebar}{RGB}{31,86,140}

\newcommand{\casetitle}[2]{%
  \par\noindent\colorbox{casetitlebar}{\parbox{\dimexpr\linewidth-2\fboxsep}{%
    \strut\color{white}\ttfamily\small\textbf{#1}\hfill\normalfont\itshape\footnotesize #2\strut}}%
  \nopagebreak\par\nointerlineskip\vspace{1pt}}

\lstdefinestyle{caseTranscript}{
  basicstyle=\ttfamily\scriptsize,
  breaklines=true,
  breakatwhitespace=false,
  columns=fullflexible,
  keepspaces=true,
  showstringspaces=false,
  backgroundcolor=\color{caseboxbg},
  frame=single,
  rulecolor=\color{caseboxrule},
  framexleftmargin=3pt, xleftmargin=3pt, xrightmargin=3pt, aboveskip=0pt,
  moredelim=[l][\color{casepromptblue}\bfseries]{[PROMPT]},
  moredelim=[l][\color{casemetagray}\itshape]{[MODEL},
  moredelim=[l][\color{casemetagray}\bfseries]{<think>},
  moredelim=[l][\color{casemetagray}\bfseries]{</think>},
  moredelim=[l][\color{caseanswergreen}\bfseries]{<answer>},
  literate={→}{{$\rightarrow$}}1 {—}{{---}}1
}

\newcounter{algorithm}

\title{Beyond Forgetting:\\
Diagnosing and Harnessing Shared Reasoning\\
in Continual RLVR}

\author{
Lirui Luo\textsuperscript{1,2} \quad
Guoxi Zhang\textsuperscript{2} \quad
Hongming Xu\textsuperscript{2} \\
Rongqing Li\textsuperscript{4} \quad
Cong Fang\textsuperscript{1,3} \quad
Lifeng Fan\textsuperscript{2} \\[0.5em]
{\normalfont\small \textsuperscript{1}State Key Lab of General AI, School of Intelligence Science and Technology, Peking University} \\
{\normalfont\small \textsuperscript{2}State Key Laboratory of General Artificial Intelligence, BIGAI} \\
{\normalfont\small \textsuperscript{3}Institute for Artificial Intelligence, Peking University} \\
{\normalfont\small \textsuperscript{4}Beijing Institute of Technology} \\
{\normalfont\small Project page: \url{https://crg-rl.github.io/}}
}

\hypersetup{
  hidelinks,
  pdftitle={Beyond Forgetting: Diagnosing and Harnessing Shared Reasoning in Continual RLVR},
  pdfauthor={Lirui Luo, Guoxi Zhang, Hongming Xu, Rongqing Li, Cong Fang, Lifeng Fan}
}

\iclrfinalcopy
\begin{document}

\maketitle

\begin{abstract}
Reinforcement learning with verifiable rewards (RLVR) commonly post-trains
reasoning models on multiple tasks, while rerunning multitask RLVR (MTRL) as
new tasks are added makes capability expansion costly. We therefore study
continual RLVR,
which updates the existing model as each task arrives.
The central question is whether a model updated this way can perform as well
as a jointly trained model. To answer this question, we
introduce
\textbf{Continual Reasoning Gym}, a continual-RLVR environment that organizes
text and visual reasoning tasks into five task sequences. In this setting,
we identify two key observations: Sequential RLVR exhibits modest forgetting,
yet its final performance remains below that of MTRL. To
understand the latter, we decompose final performance and show that forgetting
accounts for only part of the gap. To explain the former, we identify
\emph{shared reasoning}: transferable reasoning structure allows
training on one task to support others on average. We therefore
introduce \textbf{Continual Prompt Replay (CPR)}, which harnesses shared reasoning
to improve learning on the arriving and future tasks by replaying previous-task
prompts and regenerating their responses with the current policy. On average,
only CPR reaches MTRL-level performance.

\end{abstract}

\section{Introduction}

Reinforcement learning with verifiable rewards (RLVR) commonly post-trains
reasoning models on task mixtures fixed before training
\citep{reasoninggym2025,visulogic2025,agentgymrl2025,guru2025}. A deployed
coding agent, for example, may later need to adapt to newly released APIs or
changing user preferences. Rerunning multitask RLVR (MTRL) over the enlarged task mixture for every
new requirement is computationally costly. We therefore study continual RLVR,
which updates the existing model as each task arrives.

The central question is whether a model updated this way can perform as well
as a jointly trained model \citep{momeni2025}. Catastrophic forgetting is the
standard explanation for why sequential learning trails joint training
\citep{ewc2017,trace2023}. Yet reinforcement learning has recently been found
to forget less than supervised fine-tuning (SFT), suggesting that continual RLVR
may likewise exhibit less forgetting \citep{rlsrazor2025,rftpreserve2025}.

To test whether this lower forgetting allows a continually updated model to
match a jointly trained one, we introduce Continual Reasoning Gym (CRG), which
organizes text and visual reasoning tasks into task sequences. On CRG, we
identify two key observations: a sequential RLVR (Seq.\ RLVR) baseline exhibits
modest forgetting, yet its final performance remains below MTRL.

To understand the latter, we derive an exact decomposition of final performance
and show that forgetting accounts for only part of the gap. The remaining gap
reflects how well the model learns each arriving task and how earlier training
benefits future tasks. To explain the former, we identify
\emph{shared reasoning}: transferable reasoning structure allows training on one task to
benefit others on average. We provide empirical evidence for this pattern through
a task-gradient alignment analysis and a behavioral case study. We therefore
introduce Continual Prompt Replay (CPR),
which harnesses shared reasoning to improve learning on the arriving and future
tasks by bringing previous-task prompts into current-policy training
(Figure~\ref{fig:shared_reasoning_teaser}). CPR is the only continual learning
(CL) method to reach MTRL-level performance on average.

\begin{figure}[H]
  \centering
  \includegraphics[width=0.98\linewidth]{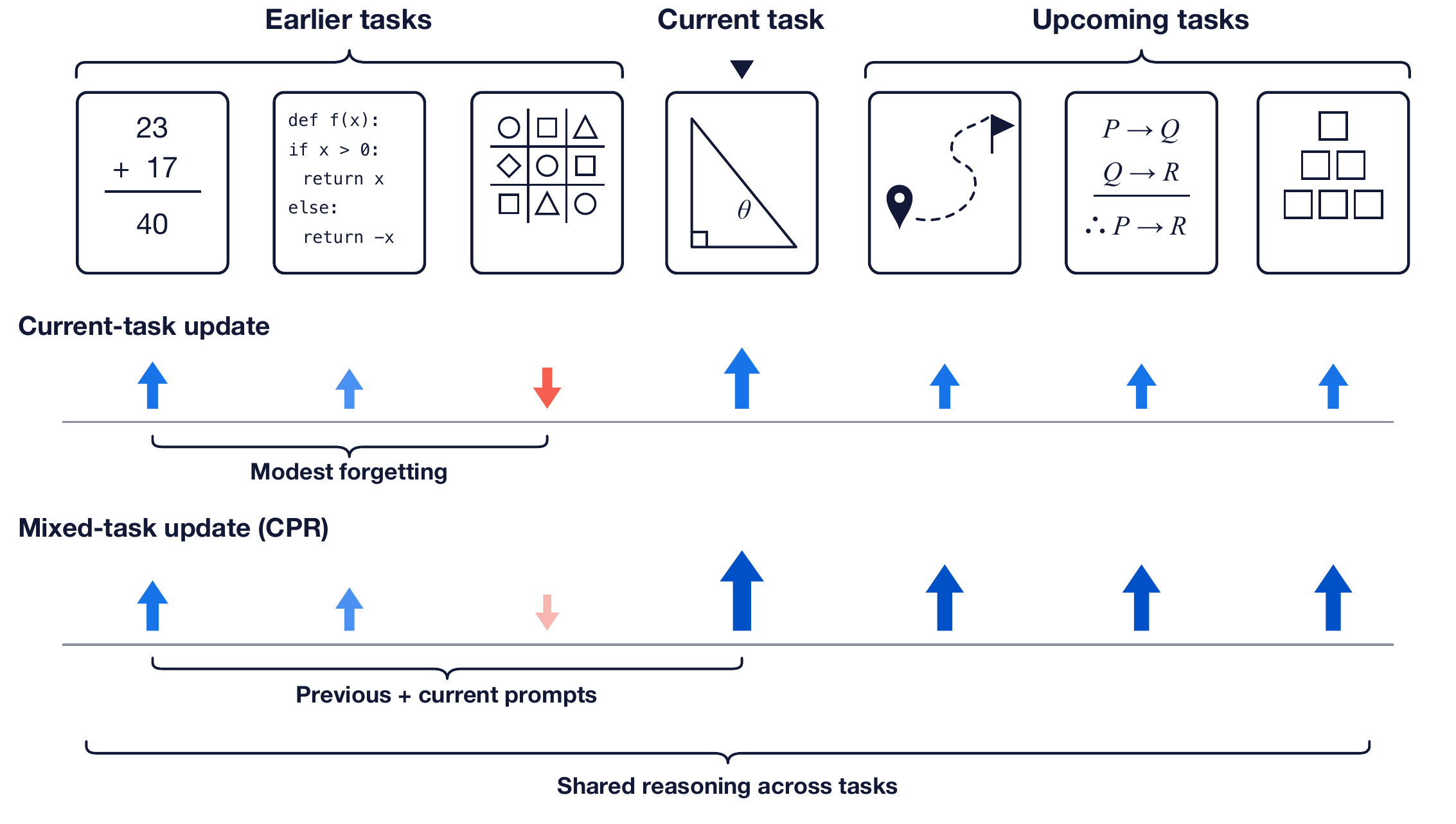}
  \caption{\textbf{Shared reasoning explains modest forgetting while CPR
  harnesses it to improve performance.}
  Later-task updates benefit earlier tasks on average, helping explain modest
  forgetting. CPR replaces some current-task prompts with previous-task prompts
  to harness shared reasoning and improve learning on the arriving and future
  tasks.}
  \label{fig:shared_reasoning_teaser}
\end{figure}

Our contributions are threefold:
\begin{enumerate}
  \item \textbf{A continual-RLVR environment.} We introduce CRG with five task
  sequences spanning text and visual reasoning,
  together with implementations of diverse CL methods.
  \item \textbf{A finding of shared reasoning.} We identify shared reasoning,
  whereby training on one task benefits others on average. It explains modest
  forgetting and can be harnessed to improve learning on the arriving and future
  tasks. We formalize this finding and provide empirical evidence from
  task-gradient measurements and a behavioral case study.
  \item \textbf{A continual-RLVR method.} We introduce CPR, which harnesses
  shared reasoning by replaying previous-task prompts and regenerating their
  responses with the current policy. CPR is the only CL method to reach
  MTRL-level performance on average.
\end{enumerate}

\section{Related Work}

\paragraph{Continual reinforcement-learning environments.}
Continual reinforcement-learning environments expose agents to ordered
sequences of interactive tasks and evaluate how learning evolves across stages.
Existing environments span robotic manipulation, embodied navigation, and
cooperative multi-agent tasks
\citep{continualworld2021,libero2023,coom2023,meal2025}. CRG applies this staged protocol to reasoning models trained with RLVR,
using programmatic verifiers to score generated responses.

\paragraph{CL benchmarks.}
CL benchmarks study sequential adaptation of large language models (LLMs),
vision--language models (VLMs), and video--language models under stage-wise
evaluation
\citep{trace2023,climb2022,vilcobench2024}. These benchmarks primarily use supervised updates on fixed input--target
pairs. CRG instead evaluates stage-wise learning and transfer under continual
RLVR, where rewards on policy-generated responses make both trajectories and
training signals policy-dependent.

Table~\ref{tab:positioning} summarizes these distinctions.
Appendix~\ref{app:extended_related_work} discusses additional related work.

\begin{table}[t]
  \centering
  \small
  \caption{\textbf{Positioning CRG at the intersection of CL and RLVR.} CRG combines reasoning
  tasks, a task sequence, RLVR training, and multiple
  CL methods.
  \checkmark{}~= present, $\times$~= absent.}
  \label{tab:positioning}
  \setlength{\tabcolsep}{4pt}
  \renewcommand{\arraystretch}{1.15}
  \begin{tabularx}{\linewidth}{X cccc}
    \toprule
    & Reasoning tasks & Task sequence & RLVR training & CL methods \\
    \midrule
    TRACE \citep{trace2023}                  & \checkmark & \checkmark & $\times$   & \checkmark \\
    CLiMB \citep{climb2022}                  & \checkmark & \checkmark & $\times$   & \checkmark \\
    Continual World \citep{continualworld2021} & $\times$ & \checkmark & $\times$ & \checkmark \\
    LIBERO \citep{libero2023}                & $\times$   & \checkmark & $\times$   & \checkmark \\
    Reasoning Gym \citep{reasoninggym2025}   & \checkmark & $\times$   & \checkmark & $\times$ \\
    VisuLogic \citep{visulogic2025}          & \checkmark & $\times$   & \checkmark & $\times$ \\
    AgentGym-RL \citep{agentgymrl2025}       & \checkmark & $\times$   & \checkmark & $\times$ \\
    \midrule
    \textbf{Continual Reasoning Gym (ours)}  & \checkmark & \checkmark & \checkmark & \checkmark \\
    \bottomrule
  \end{tabularx}
\end{table}

\section{Preliminaries}\label{sec:prelim}

\noindent\textbf{Continual RL.}\spacefactor=1000\space\space%
A continual RL problem chains an ordered sequence of tasks
$\tau_1,\dots,\tau_T$, each modeled as an MDP, that a single agent learns one
at a time \citep{continualworld2021}. This paper studies the RLVR
specialization, where each task carries a programmatic verifier and every
training stage consists of RLVR updates. Starting from a base policy $M_0$,
stage $i$ trains on task $\tau_i$ and produces $M_i$ from $M_{i-1}$, and the
goal is a final policy with the highest average success across the $T$ tasks.

\noindent\textbf{Token-level Markov Decision Process.}\spacefactor=1000\space\space%
Task $\tau_i$ provides a prompt distribution $\mathcal{D}_i$ and a programmatic
verifier $r_i(x,y)\in\{0,1\}$ for a prompt $x\sim\mathcal{D}_i$ and completed
response $y\in\mathcal{V}^*$. We model response generation on $\tau_i$ as a
token-level MDP $(\mathcal{S},\mathcal{A},P,\bar r_i,\rho_i)$
\citep{suttonbarto2018}. Its state
space $\mathcal{S}$ contains prompt-conditioned token prefixes, its action space
$\mathcal{A}=\mathcal{V}$ is the vocabulary including the terminal token
$\texttt{eos}$, and its initial-state distribution $\rho_i$ draws a prompt
$x\sim\mathcal{D}_i$ as $s_1=x$. The deterministic transition appends the
sampled token,
$s_{t+1}=P(s_t,y_t)=\operatorname{concat}(s_t,y_t)$. The token-level reward
$\bar r_i$ is zero before termination and equals $r_i(x,y)$ when
$\texttt{eos}$ completes response $y$. Thus an undiscounted rollout sampled
from $\pi_\theta$ has return $r_i(x,y)$, and the RLVR objective is
\begin{equation}\label{eq:rlvr-objective}
J_i(\theta)=\mathbb{E}_{x\sim\mathcal{D}_i,\,y\sim\pi_\theta(\cdot\mid x)}[r_i(x,y)].
\end{equation}

\FloatBarrier
\section{Continual Reasoning Gym}\label{sec:benchmark}

In this section, we instantiate the continual RLVR setting from
Section~\ref{sec:prelim} with five sequences of verifiable reasoning tasks.

\begin{figure}[t]
  \centering
  \includegraphics[width=0.98\linewidth]{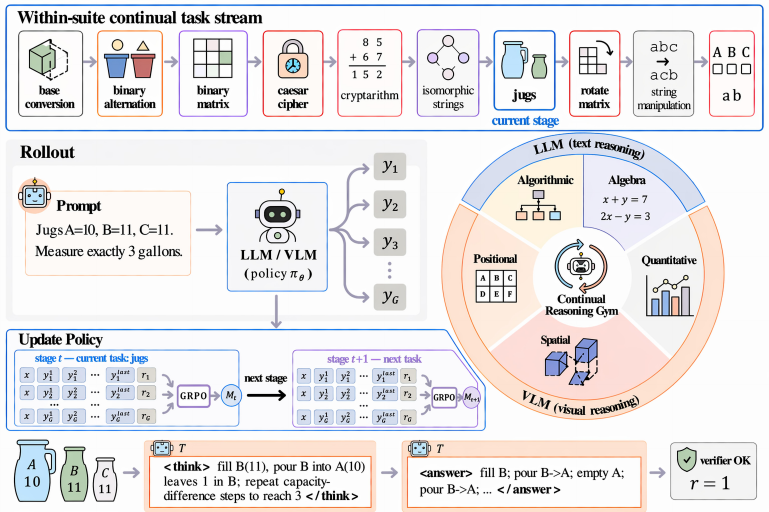}
  \caption{\textbf{Policy training proceeds stage by stage through a sequence
  of verifiable reasoning tasks.} \emph{Top}: the LLM algorithmic sequence contains
  $T{=}10$ tasks, with jugs as the current task. \emph{Left}: at each stage, the
  policy rolls out responses $y_1,\dots,y_G$ to a prompt. A verifier scores the
  responses, and Group Relative Policy Optimization (GRPO) uses these scores to
  produce the stage policy $M_t$.
  \emph{Wheel}: the five task sequences cover two text-reasoning settings and
  three visual-reasoning settings. \emph{Bottom}: a jugs example shows the
  policy's reasoning and verifier-checked answer.}
  \label{fig:benchmark_schematic}
\end{figure}

\subsection{Task Sequences}
\label{subsec:task-streams}

A task sequence turns a static task pool into a staged, non-stationary training
stream. For each sequence, we select several related but distinct tasks and
train them sequentially, allocating $N/T$ steps to each under a total budget of
$N$ steps (Figure~\ref{fig:benchmark_schematic}). We instantiate five task sequences across
two modalities. The two text settings draw
from Reasoning Gym's largest category \texttt{algorithmic} and its dedicated
\texttt{algebra} category, covering procedural symbolic manipulation and
symbolic mathematics \citep{reasoninggym2025}. The three visual settings take
VisuLogic's three largest reasoning categories, namely quantitative, spatial,
and positional \citep{visulogic2025}. Within each category, we define the
ordered stage-level subtypes
that form the corresponding sequence. Appendix~\ref{app:task_streams} reports
dataset statistics, and Appendix~\ref{app:task_examples} provides
representative examples.

\subsection{Baseline Algorithms}
\label{subsec:baselines}

We use Seq.\ RLVR as the default baseline. At each stage, it continues
training the policy with standard RLVR on the arriving task. MTRL jointly
optimizes the same task set and serves as the matched multitask reference.
Five CL interventions extend Seq.\ RLVR. Elastic Weight Consolidation (EWC)
\citep{ewc2017} penalizes changes to important weights. KL-to-old-policy
penalizes divergence from the pre-task policy. Orthogonal Subspace Fine-Tuning
(OSFT) \citep{osft2025} constrains updates to the orthogonal complement of
critical directions. ReDo \citep{redo2023} reinitializes dormant units, while
FIRE \citep{fire2026} uses Frobenius-isometry reinitialization to balance
stability and plasticity. We additionally include Muon \citep{muon2025}, which
orthogonalizes momentum updates, as an optimizer control.

\subsection{Evaluation Quantities}\label{sec:metrics}

The policy is evaluated on every task before training and after each task stage.
From these evaluations, we report BaseAvg and FinalAvg for the two endpoints.
To characterize what happens between them, we use forward transfer (FWT),
task-learning gain (TLG), and backward transfer (BWT). FWT captures the effect
of earlier-task training before a task's own stage, TLG the effect of direct
training on that task, and BWT the effect of subsequent training on later tasks.
To compare final performance with the matched multitask reference across task
sequences, we report the continual-to-multitask ratio (CTM):\label{sec:closeness}
\begin{equation}\label{eq:ctm}
\mathrm{CTM}=\frac{\mathrm{FinalAvg}}{\mathrm{FinalAvg}_{\mathrm{MTRL}}},
\end{equation}
where $\mathrm{FinalAvg}_{\mathrm{MTRL}}$ is the multitask reference trained
on the same task set. A CTM of $1$ indicates MTRL-level performance.
Further details are provided in
Appendix~\ref{app:metric-definitions}.

\FloatBarrier
\FloatBarrier
\section{Continual Prompt Replay}\label{sec:method}

In this section, we first decompose final performance to show why the
MTRL gap remains despite modest forgetting. We then formalize shared
reasoning to explain modest forgetting and motivate CPR, which brings
previous-task prompts into current-policy training to improve FWT and TLG.

\subsection{Decomposing the MTRL Gap}\label{subsec:where-gap}

\paragraph{Seq.\ RLVR remains below MTRL despite modest forgetting.}
We first test whether forgetting accounts for the gap. Under Seq.\ RLVR,
training on later tasks reduces performance on previously learned tasks by only
$2.47$ percentage points on average across the five task sequences, yet final
performance reaches only $88\%$ of MTRL (Figure~\ref{fig:transfer_corr}).
To determine what accounts for the gap beyond this modest forgetting, we
decompose $\mathrm{FinalAvg}$ as follows.

\begin{proposition}[Final Performance Beyond Forgetting]\label{prop:decomposition}
\begin{equation}\label{eq:decomposition}
\mathrm{FinalAvg}=\mathrm{BaseAvg}
+\tfrac{T-1}{T}\,\mathrm{FWT}
+\mathrm{TLG}
+\tfrac{T-1}{T}\,\mathrm{BWT}.
\end{equation}
\end{proposition}
\begin{proof}
See Appendix~\ref{app:decomposition-proof}.
\end{proof}

\begin{figure}[H]
  \centering
  \includegraphics[width=0.95\linewidth]{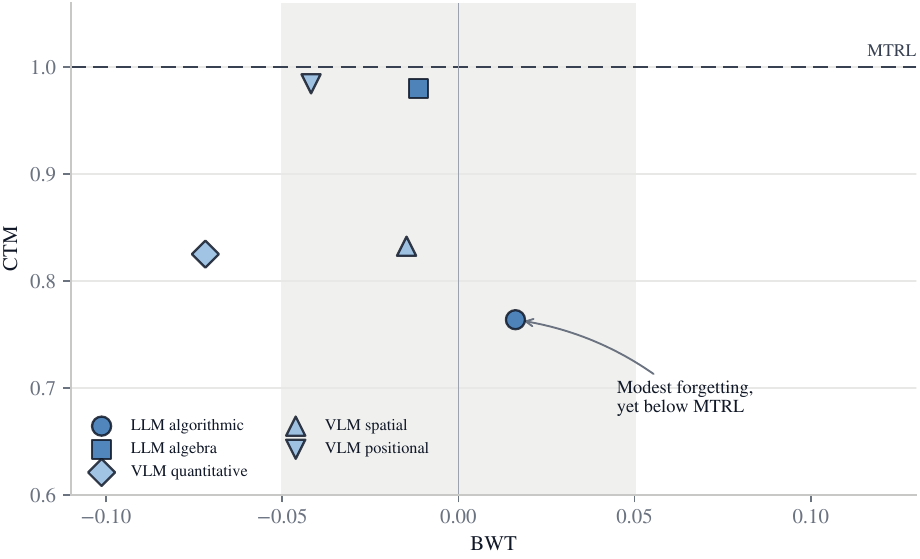}
  \caption{\textbf{Seq.\ RLVR remains below MTRL despite modest
  forgetting.} Each marker represents one task sequence. Across task sequences,
  performance on earlier tasks falls by only $2.47$ percentage points on
  average, yet final performance reaches $88\%$ of MTRL.}
  \label{fig:transfer_corr}
\end{figure}

Since all methods within a task sequence start from the same initial policy, $\Delta\mathrm{BaseAvg}=0$. Thus,
\begin{equation}\label{eq:matched-finalavg-delta}
\Delta\mathrm{FinalAvg}
=\tfrac{T-1}{T}\Delta\mathrm{FWT}
+\Delta\mathrm{TLG}
+\tfrac{T-1}{T}\Delta\mathrm{BWT}.
\end{equation}
Even if the weighted BWT contribution were zero, the gap to MTRL would
remain. We therefore ask two questions: why Seq.\ RLVR exhibits modest
forgetting, and how the remaining FWT and TLG gaps can be reduced.

\subsection{Shared Reasoning Explains Modest Forgetting}\label{subsec:shared-reasoning}

To address the first question, we analyze how later-task updates affect
earlier-task objectives. Consider an update at stage $t$ with parameters
$\theta$. Let $J_i$ denote the objective for task $i$, with
gradient $g_i^{(t)}=\nabla_\theta J_i(\theta)$, and let $g_k^{(t)}$ denote the
update direction estimated from task $k$. After an update of size $\eta$ along
this direction,
\begin{equation*}
J_i(\theta+\eta g_k^{(t)})-J_i(\theta)
=\eta\langle g_i^{(t)},g_k^{(t)}\rangle+O(\eta^2).
\end{equation*}

The sign of the inner product indicates whether an update from task $k$
locally benefits or harms task $i$. The modest forgetting observed above
suggests that the cumulative effect of later-task updates is not strongly
harmful to earlier tasks on average. One
possible explanation is that reasoning tasks share reasoning structure, thus
allowing updates on one task to benefit others on average. We call this pattern
shared reasoning and formalize it as follows.

\begin{assumption}[Mean task-gradient alignment]\label{asm:shared-reasoning}
Across reasoning tasks,
\begin{equation}\label{eq:mean-gradient-alignment}
\mathbb E_{t,i,k}\!\left[
\langle g_i^{(t)},g_k^{(t)}\rangle\right]\geq 0.
\end{equation}
\end{assumption}

For Seq.\ RLVR, every update is drawn from the current task, so $k=t$.
Restricting to earlier tasks $i<t$ and averaging over BWT's stages and tasks
gives the following corollary.

\begin{corollary}[Backward Transfer under Mean Alignment]\label{cor:low-forgetting}
Under Assumption~\ref{asm:shared-reasoning}, later-stage Seq.\ RLVR updates
have a non-negative expected first-order contribution to earlier-task
objectives.
\end{corollary}
\begin{proof}
See Appendix~\ref{app:corollary-proofs}.
\end{proof}

Corollary~\ref{cor:low-forgetting} connects mean alignment to modest
forgetting: later-task updates can benefit earlier-task objectives on average.
We next ask whether shared reasoning can also improve FWT and TLG. Since mean
alignment does not require every task pair to align, the current-task direction
may harm some task objectives, while previous-task directions may complement
it. This motivates CPR.

\subsection{Harnessing shared reasoning with CPR}

To bring previous-task directions into the current update, CPR replays
previous-task prompts and regenerates their responses with the current policy.
CPR stores each prompt with its task identity and latest pass rate. During later
stages, it selects previous-task prompts with intermediate pass rates to replace
a fraction of current-task prompts, changing the task-sampling distribution
without increasing the number of prompts or rollouts per update.

Let $\widehat\rho_t$ be the fraction of replayed prompts in an update at stage
$t$, and let $q_{t,k}$ be the fraction of those prompts drawn from task $k<t$.
The mixed-batch update direction is
\begin{equation}\label{eq:cpr-mixed-gradient}
g_t^{\mathrm{CPR}}=(1-\widehat\rho_t)g_t^{(t)}
+\widehat\rho_t\sum_{k<t}q_{t,k}g_k^{(t)}.
\end{equation}
Applying Assumption~\ref{asm:shared-reasoning} to this mixture yields the
following corollary.

\begin{corollary}[CPR under Mean Alignment]
\label{cor:joint-improvement}
Under Assumption~\ref{asm:shared-reasoning}, CPR updates have a non-negative
expected first-order contribution when averaged over the relevant stages and
task objectives.
\end{corollary}
\begin{proof}
See Appendix~\ref{app:corollary-proofs}.
\end{proof}

\FloatBarrier
\section{Experiments}\label{sec:experiments}

We first test the shared-reasoning explanation for modest forgetting, then
evaluate whether CPR harnesses shared reasoning to improve FWT and TLG.

\subsection{Setup}\label{subsec:experimental-setup}
We instantiate CRG with Qwen3-4B in two text settings and
Qwen2.5-VL-7B in three visual settings. For each setting, every method uses a
budget of $N=500$ optimizer steps. Appendix~\ref{app:hyperparameters} provides
further experimental details.

\subsection{Evidence for Shared Reasoning}

To evaluate shared reasoning as an explanation for modest forgetting, we
measure task-gradient alignment and trace changes in reasoning behavior across
stages.

\paragraph{Task gradients are positively aligned on average.}
On the ten-task LLM algorithmic sequence, we compute pairwise cosine
similarities between task gradients under Seq.\ RLVR. The mean cosine is
$0.345$, and $82.2\%$ of task pairs have positive cosine
(Figure~\ref{fig:grad_cosine}). This positive mean alignment is consistent with
the shared-reasoning explanation for modest forgetting. Additional gradient
diagnostics are provided in Appendix~\ref{app:gradient-diagnostics}.

\begingroup
\setlength{\intextsep}{8pt}
\begin{figure}[H]
  \centering
  \includegraphics[width=0.88\linewidth]{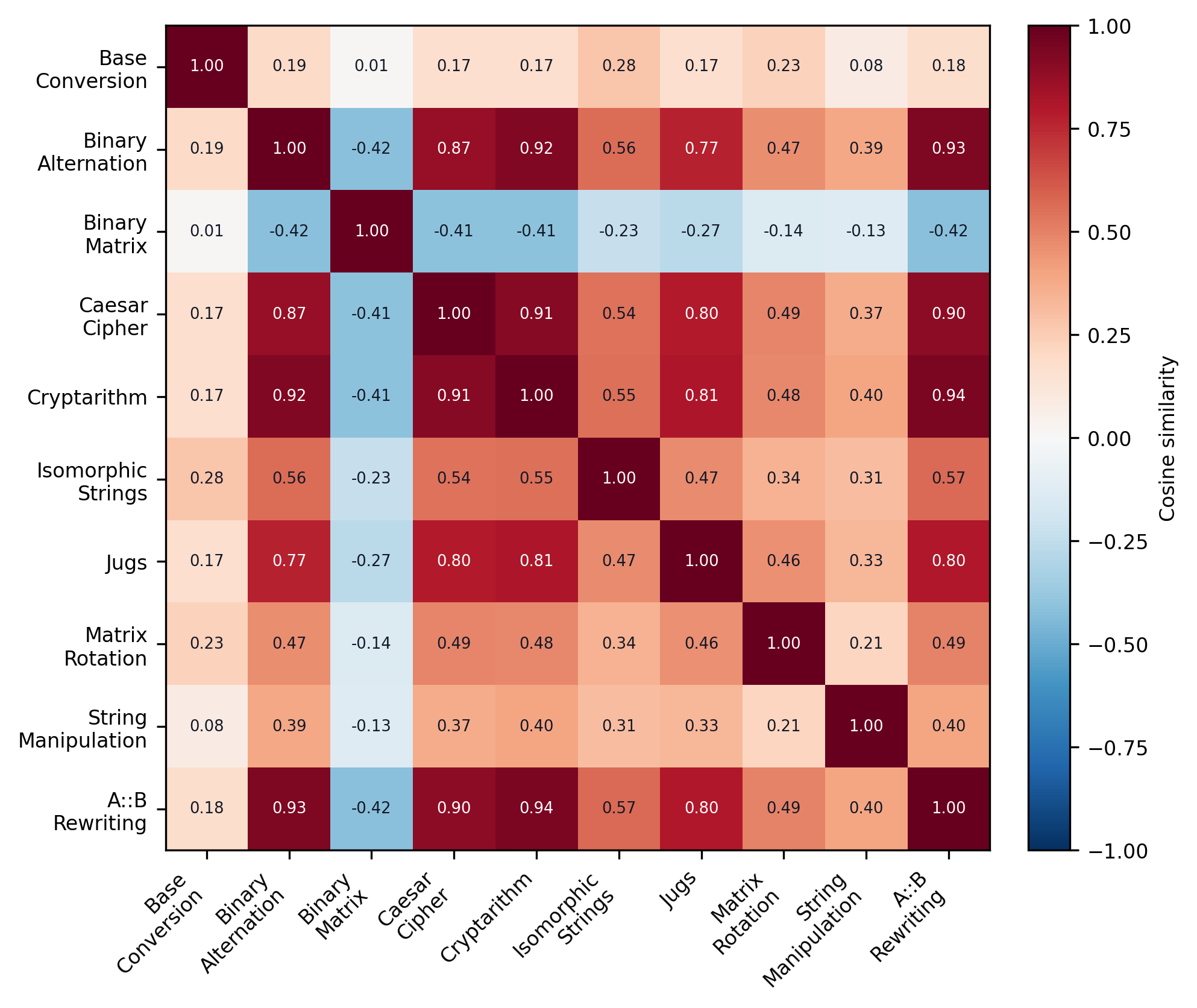}
  \caption{\textbf{Task gradients are positively aligned on average under
  Seq.\ RLVR.} The heatmap shows pairwise cosine similarities between task
  gradients on the ten-task LLM algorithmic sequence. The mean cosine is
  $0.345$, and $82.2\%$ of task
  pairs have positive cosine.}
  \label{fig:grad_cosine}
\end{figure}
\endgroup

\paragraph{Shared reasoning appears in model behavior.}
Having found positive mean alignment between task gradients, we next examine
whether training on other tasks produces corresponding changes in performance
and reasoning. Figure~\ref{fig:reasoning_chain_case_study} tracks jugs success
and representative reasoning excerpts throughout Seq.\ RLVR. Jugs performance
improves both before its own stage and during later-task training, with the
later-task gain directly consistent with shared reasoning as an explanation for
modest forgetting. The excerpts show a parallel progression. Initially,
the model finds an optimal route but repeatedly second-guesses it and emits no
answer. Before jugs training, it takes a longer route without answering. After direct training, it
returns a valid answer after rechecking. At the end of the sequence, it reaches
the same answer through shorter, more direct reasoning. Full transcripts are
provided in Appendix~\ref{app:reasoning_chain_case_study}.

\begin{figure}[t]
  \centering
  \includegraphics[width=\linewidth]{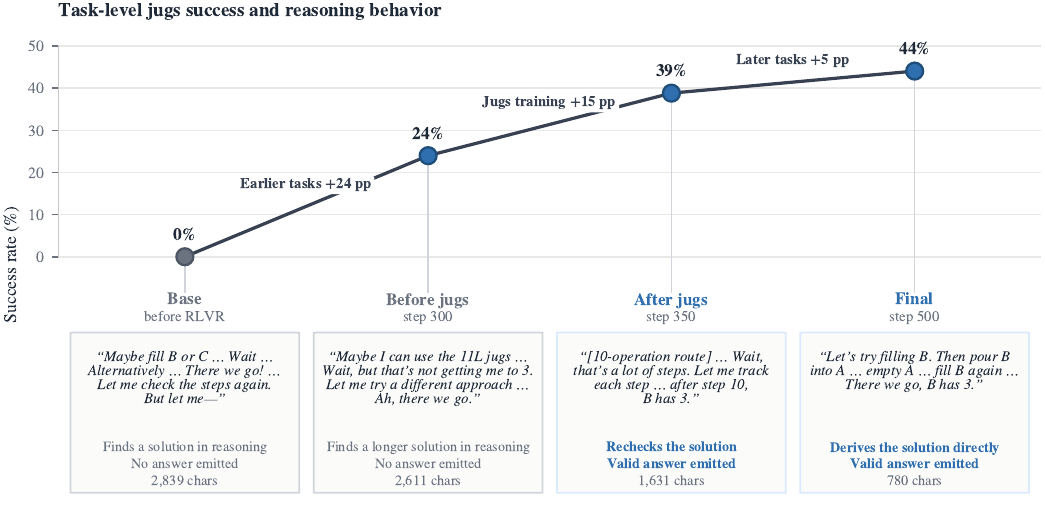}
  \caption{\textbf{A jugs case illustrates shared reasoning in both
  performance and reasoning behavior.} Jugs performance improves before its
  own stage and again during later-task training. The aligned excerpts show a
  progression from unresolved reasoning to valid and more direct solutions.}
  \label{fig:reasoning_chain_case_study}
\end{figure}

We therefore observe shared reasoning at both the update and behavioral
levels. We next test whether CPR can harness it to improve learning on the
arriving and future tasks.

\subsection{CPR Results}

\paragraph{CPR reaches MTRL-level performance on average.}
CPR achieves a mean CTM of $1.03$, up from $0.88$ under Seq.\ RLVR, and is the
only evaluated CL method to reach MTRL-level performance on average
(Table~\ref{tab:family_signatures}).

\paragraph{CPR improves forward transfer and direct task learning.}
Relative to Seq.\ RLVR, CPR raises mean FinalAvg by $6.4$ percentage points.
Of this gain, FWT accounts for $+2.4$ points and TLG for $+4.9$ points, while
BWT accounts for $-0.9$ points. CPR's gain therefore comes from improved
learning on arriving and future tasks (Table~\ref{tab:family_signatures}).

\begin{table}[H]
  \centering
  \caption{\textbf{CPR is the only method to reach MTRL-level performance
  on average.} For each method, FWT, TLG, and BWT report their percentage-point
  contributions to the change in $\mathrm{FinalAvg}$ relative to Seq.\ RLVR.
  Total is their sum.}
  \label{tab:family_signatures}
  \small
  \setlength{\tabcolsep}{2.4pt}
  \begin{tabular}{lrrrrrrrr}
    \toprule
    Metric & \shortstack{Seq.\\RLVR} & \shortstack{Reset\\(ReDo)} & \shortstack{Reset\\(FIRE)} & \shortstack{Regularization\\(EWC)} & \shortstack{Optimizer\\(Muon)} & \shortstack{Isolation\\(OSFT)} & \shortstack{Regularization\\(KL)} & \textbf{\shortstack{Replay\\(CPR)}} \\
    \midrule
    FWT & 0.0 & -1.5 & -5.8 & -1.7 & +1.3 & -3.4 & +0.6 & \textbf{+2.4} \\
    TLG & 0.0 & -7.8 & -4.0 & +0.4 & -0.8 & -7.9 & -5.0 & \textbf{+4.9} \\
    BWT & 0.0 & +3.0 & -14.5 & +2.7 & +2.5 & +2.7 & \textbf{+5.4} & -0.9 \\
    Total & 0.0 & -6.3 & -24.4 & +1.4 & +3.0 & -8.5 & +1.1 & \textbf{+6.4} \\
    \midrule
    Mean CTM & 0.88 & 0.78 & 0.29 & 0.87 & 0.97 & 0.65 & 0.95 & \textbf{1.03} \\
    \bottomrule
  \end{tabular}
\end{table}

\paragraph{Previous-task gradients better align updates with the full task set.}
To examine their role in CPR's gains, we evaluate Qwen3-4B task gradients and compare each current-task gradient with an equally weighted
combination of the current-task gradient and the mean previous-task gradient.
This combination increases cosine similarity with the mean gradient across all
tasks from $0.43$ to $0.79$ on average and yields higher alignment
at all nine stages (Figure~\ref{fig:cpr-gradient-alignment}). Additional gradient
diagnostics are provided in Appendix~\ref{app:gradient-diagnostics}.

\begin{figure}[t]
  \centering
  \includegraphics[width=0.48\linewidth]{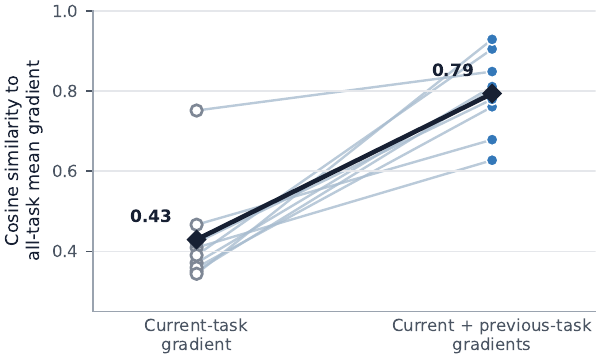}
  \caption{\textbf{Previous-task gradients better align updates with the
  full task set.} For Qwen3-4B, each line compares a current-task
  gradient with an equal-weight combination of that gradient and the mean
  previous-task gradient. Dark diamonds show the averages across the nine
  stages.}
  \label{fig:cpr-gradient-alignment}
\end{figure}

To determine whether CPR benefits from replay itself or from regenerating
responses with the current policy, we compare it with no replay and sample
replay.

\paragraph{Current-policy regeneration drives the replay gain.}
\label{sec:ablations}
We replace CPR's freshly generated responses with trajectories from an earlier
policy. On LLM algorithmic reasoning, this sample-replay variant reaches a
FinalAvg of $47.5\%$, below no replay at $49.9\%$, whereas CPR reaches $63.3\%$
and approaches MTRL at $65.3\%$ (Figure~\ref{fig:cpr_resampling_ablation}). The
gain therefore depends on regenerating responses with the current policy.

\begin{figure}[H]
  \centering
  \includegraphics[width=0.48\linewidth]{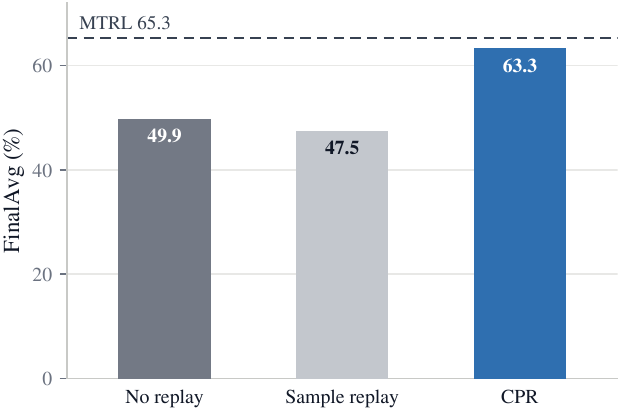}
  \caption{\textbf{Current-policy regeneration drives CPR's replay gain.}
  On LLM algorithmic reasoning, CPR approaches MTRL. Sample replay reuses
  trajectories generated by an earlier policy and underperforms no replay.}
  \label{fig:cpr_resampling_ablation}
\end{figure}

This comparison shows that CPR's gain depends on both previous-task prompts and
current-policy regeneration. Appendix~\ref{app:cpr-ablations} reports additional
replay-pool comparisons, and Appendix~\ref{app:cpr-ratio-robustness} reports
replay-ratio robustness.

\FloatBarrier
\section{Conclusion}

In this paper, we introduce CRG to evaluate whether continual RLVR can reach
MTRL-level performance across text and visual reasoning task sequences. Using CRG, we identify two key observations:
Seq.\ RLVR forgets modestly, yet its performance remains below MTRL. To
understand the latter, we decompose final performance and show that preserving
earlier-task performance alone cannot close the gap. To understand the former,
we identify shared reasoning, whereby training on one task benefits others on
average. We then instantiate this insight as CPR, which replays previous-task
prompts and regenerates responses with the current policy to harness shared
reasoning and improve learning on arriving and future tasks. On CRG, CPR is the
only evaluated CL method to reach MTRL-level performance on average.

\noindent\textbf{Limitations.}\spacefactor=1000\space\space%
Our experiments focus on five task sequences using Qwen3-4B and Qwen2.5-VL-7B.
Extending CRG to agentic settings and larger models remains future work. Our
analysis characterizes local update effects and does not capture nonlinear
dynamics over extended training.

\FloatBarrier
\subsection*{AI use statement}

In this work, we use AI for formulating mathematical claims,
providing critical ingredients for proving mathematical claims, assisting in
the writing of proofs, refining hypotheses, implementing methods, and
reformatting datasets. We have not used AI for developing
theoretical models or conceptual frameworks, designing research methodology or
experiments, generating synthetic datasets, assisting with translation,
supporting qualitative and thematic data analysis, or interpreting results,
and the remaining required disclosure tasks are not applicable to this work.
Additionally, we use AI for creating and modifying scientific
figures, creating and editing software code, drafting and editing parts of the
paper, summarizing existing literature, sourcing information, identifying
relevant literature, and suggesting the paper's structure. We have reviewed all
AI-assisted work. We checked mathematical arguments, verified citations against
their sources, tested AI-assisted code through the reported workflows, and
reviewed AI-assisted figures and manuscript revisions. We take responsibility for the final
content of this work, including text, claims or artifacts produced with the aid
of AI.

\subsection*{Ethics statement}

This work does not involve private data or raise ethical concerns.

\subsection*{Reproducibility statement}

Section~\ref{sec:benchmark} defines the task sequences and stage protocol, while
Section~\ref{subsec:experimental-setup} specifies training and evaluation.
Appendix~\ref{app:task_streams} lists the task sequences, dataset statistics,
task construction, and VLM provenance. The rest of the appendix provides exact
metric definitions and boundary conventions, complete hyperparameters, full
results, proofs, CPR details and ablations, gradient diagnostics, and
raw case-study transcripts. Code, configurations, evaluation records, and
scripts for regenerating the reported tables and figures will accompany the
paper release.

\bibliographystyle{style/iclr2027_conference}
\bibliography{bib/references}

\clearpage
\appendix
\makeatletter
\providecommand{\protected@file@percent}{}
\def\addcontentsline#1#2#3{%
  \begingroup
    \let\label\@gobble
    \@ifundefined{@currentHref}{\def\@currentHref{}}{}%
    \addtocontents{#1}{\protect\contentsline{#2}{#3}{\thepage}{\@currentHref}\protected@file@percent}%
  \endgroup}
\makeatother
\section*{\centering\LARGE Appendix}

\etocsetnexttocdepth{subsection}
\etocsetlocaltop{part}
\etocsettocstyle{\subsection*{Appendix Contents}}{}
\localtableofcontents

\section{Continual Reasoning Gym Details}

\subsection{Task Sequences}\label{app:task_streams}

CRG instantiates five task sequences comprising 30 stages.
The text tasks draw from ReasoningGym \citep{reasoninggym2025}, and the visual
tasks draw from VisuLogic \citep{visulogic2025}.

\begin{table}[t]
\centering
\small
\caption{\textbf{Five task sequences define the Continual Reasoning Gym streams.} Arrows show the training sequence. The LLM algorithmic, LLM algebra, VLM quantitative, VLM spatial, and VLM positional streams contain 10, 6, 4, 6, and 4 stages, respectively.}
\label{tab:task_stream_stages}

\begin{tabularx}{\textwidth}{>{\raggedright\arraybackslash}p{0.09\textwidth}>{\raggedright\arraybackslash}p{0.17\textwidth}>{\raggedright\arraybackslash}X}
\toprule
Modality & Stream & Ordered stages \\
\midrule
LLM & Algorithmic Reasoning & Base Conversion $\rightarrow$ Binary Alternation $\rightarrow$ Binary Matrix $\rightarrow$ Caesar Cipher $\rightarrow$ Cryptarithm $\rightarrow$ Isomorphic Strings $\rightarrow$ Jugs $\rightarrow$ Matrix Rotation $\rightarrow$ String Manipulation $\rightarrow$ A::B Rewriting \\
LLM & Algebra Reasoning & Complex Arithmetic $\rightarrow$ Intermediate Integration $\rightarrow$ Polynomial Equations $\rightarrow$ Polynomial Multiplication $\rightarrow$ Simple Equations $\rightarrow$ Simple Integration \\
VLM & Quantitative Reasoning & Linear Quantity $\rightarrow$ Prime Quantity $\rightarrow$ Point Quantity $\rightarrow$ Counter Quantity \\
VLM & Spatial Reasoning & Assembly $\rightarrow$ Cross-sectional $\rightarrow$ Hexahedron $\rightarrow$ Three Views $\rightarrow$ Spatial Orders $\rightarrow$ Polyhedron Rotation \\
VLM & Positional Reasoning & Translation $\rightarrow$ Rotation $\rightarrow$ Comparative $\rightarrow$ Flip \\
\bottomrule
\end{tabularx}
\end{table}

\begin{table}[t]
  \centering
  \small
  \setlength{\tabcolsep}{5.5pt}
  \caption{\textbf{Continual Reasoning Gym contains 30 stages across five
  task sequences.} Each setting is split $8{:}2$ into training and test sets.
  LLM examples are generated procedurally, with at most 20{,}000 total examples
  sampled for each text setting. Symbolic and four-way answers are scored by
  the corresponding task verifier.}
  \label{tab:dataset_statistics}
  \begin{tabular}{llrrl}
    \toprule
    Modality & Setting & Stages & \shortstack{Total\\examples} & Answer format \\
    \midrule
    LLM & Algorithmic & 10 & 20{,}000 & Symbolic \\
    LLM & Algebra & 6 & 20{,}000 & Symbolic \\
    VLM & Quantitative & 4 & 1{,}386 & Four-way choice \\
    VLM & Spatial & 6 & 1{,}043 & Four-way choice \\
    VLM & Positional & 4 & 743 & Four-way choice \\
    \midrule
    \multicolumn{2}{l}{\textbf{Total}} & \textbf{30} & --- & --- \\
    \bottomrule
  \end{tabular}
\end{table}

VisuLogic defines the top-level quantitative, spatial, and positional domains.
We partition each domain into project-defined stage subtypes and fix their
training order in Table~\ref{tab:task_stream_stages}. GPT-5.5 using the
\texttt{xhigh} reasoning-effort setting generated the initial subtype assignments.
Professional researchers then manually reviewed every assignment. Three
quantitative examples remained ambiguous after review and are excluded from the
four stage pools. The researcher-reviewed labels define the final spatial and
positional subtypes.

\subsubsection{LLM Reasoning}

\paragraph{Algorithmic Reasoning (LLM).}
This setting covers ten procedural tasks whose symbolic answers require explicit
multi-step computation. Their order is:
\begin{description}[leftmargin=1.5em,itemsep=1pt,topsep=2pt,parsep=0pt]
  \item[Base Conversion] rewrites a number between numeral bases.
  \item[Binary Alternation] returns the minimum number of swaps that make a binary string alternate, or reports impossibility.
  \item[Binary Matrix] computes, for every cell of a binary matrix, the Manhattan distance to the nearest zero.
  \item[Caesar Cipher] recovers the plaintext of a shift-ciphered message whose shift is unknown.
  \item[Cryptarithm] solves an addition puzzle for the unique letter-to-digit assignment.
  \item[Isomorphic Strings] decides whether two strings admit a one-to-one character mapping that preserves structure.
  \item[Jugs] produces a sequence of fill, pour, and empty operations that measures a target volume.
  \item[Matrix Rotation] rotates a matrix by a specified angle.
  \item[String Manipulation] applies a described sequence of string edits.
  \item[A::B Rewriting] reduces a program under the A::B rewriting system to its normal form.
\end{description}
Each sub-task carries an exact-match verifier over its symbolic answer.

\paragraph{Algebra Reasoning (LLM).}
This setting covers six symbolic-mathematics tasks. Their order is:
\begin{description}[leftmargin=1.5em,itemsep=1pt,topsep=2pt,parsep=0pt]
  \item[Complex Arithmetic] evaluates arithmetic on complex numbers.
  \item[Intermediate Integration] evaluates integrals at a harder difficulty level.
  \item[Polynomial Equations] solves a polynomial for its roots.
  \item[Polynomial Multiplication] expands a product of polynomials.
  \item[Simple Equations] solves a basic algebraic equation for its unknown.
  \item[Simple Integration] evaluates integrals at an easier difficulty level.
\end{description}
Answers are checked symbolically, so numerically equivalent forms are accepted.

\subsubsection{VLM Visual Reasoning}

\paragraph{Quantitative Reasoning (VLM).}
This setting tests changes in graphical quantities, including counts of points,
lines, and angles and arithmetic relations among shapes. Its four stages are
project-defined subtypes of VisuLogic's Quantitative domain:
\begin{description}[leftmargin=1.5em,itemsep=1pt,topsep=2pt,parsep=0pt]
  \item[Linear Quantity] element counts change linearly along a row, column, or sequence, by a constant or position-indexed step.
  \item[Prime Quantity] counts follow a number-theoretic rule such as primality, parity, multiples, or divisibility.
  \item[Point Quantity] the changing quantity is points, intersections, endpoints, or nodes.
  \item[Counter Quantity] other count patterns over elements, blocks, edges, segments, or black and white cells.
\end{description}

\paragraph{Spatial Reasoning (VLM).}
This setting requires reconstructing three-dimensional structure from
two-dimensional figures, including folding surfaces and integrating solid
shapes. Its six stages are project-defined subtypes of VisuLogic's Spatial
domain:
\begin{description}[leftmargin=1.5em,itemsep=1pt,topsep=2pt,parsep=0pt]
  \item[Assembly] assembling, completing, or decomposing 3D parts into matching polyhedron pieces.
  \item[Cross-sectional] identifying the cross-section of a 3D solid cut by a plane.
  \item[Hexahedron] reconstructing a cube or cuboid surface from its net or folded box.
  \item[Three Views] matching a 3D object to its orthographic projection views.
  \item[Spatial Orders] a sequence, classification, or ordering pattern over spatial figures.
  \item[Polyhedron Rotation] comparing a solid under 3D rotations, viewing orientations, or symmetry.
\end{description}

\paragraph{Positional Reasoning (VLM).}
This setting examines transformations that relocate objects while preserving
their constituent elements. Its four stages are project-defined subtypes of
VisuLogic's Positional domain:
\begin{description}[leftmargin=1.5em,itemsep=1pt,topsep=2pt,parsep=0pt]
  \item[Translation] marked elements shift, swap, or reorder positions while keeping their orientation.
  \item[Rotation] elements or orientations rotate or cycle around a center or ordered positions.
  \item[Comparative] a mixed positional relationship in which no single translation, rotation, or flip dominates.
  \item[Flip] the pattern uses mirror reflection across an axis.
\end{description}
All three VLM settings use four-way multiple-choice instances scored by exact
match.

\clearpage
\subsection{Evaluation Metrics}\label{app:metric-definitions}

Stage-wise evaluation records one score for every stage policy and task. Let
$R_{j,i}=S(M_j,\tau_i)$ denote the mean verifier success rate of stage-$j$
policy $M_j$ on the fixed evaluation examples for task $\tau_i$. Evaluating
$j=0,\ldots,T$ on every $i=1,\ldots,T$ gives a $(T{+}1)\times T$ matrix $R$
for each method and task sequence. This evaluation follows GEM \citep{gem2017}, its
continual-RL use in Continual World \citep{continualworld2021}, and its
large-model use in TRACE \citep{trace2023}.

Four entries identify the task-level performance values used by these metrics:
\begin{equation}\label{eq:metric-task-points}
B_i=R_{0,i},\qquad P_i=R_{i-1,i},\qquad
A_i=R_{i,i},\qquad F_i=R_{T,i}.
\end{equation}
They denote performance before any training, immediately before and after
direct training on task $i$, and after all stages. The boundary conventions are
$P_1\defeq B_1$ and $A_T\defeq F_T$.

The metrics are
\begin{equation}\label{eq:metric-aggregates}
\begin{aligned}
\mathrm{BaseAvg} &= \frac{1}{T}\sum_{i=1}^{T}B_i,
&\qquad
\mathrm{FinalAvg} &= \frac{1}{T}\sum_{i=1}^{T}F_i,\\
\mathrm{FWT} &= \frac{1}{T-1}\sum_{i=2}^{T}(P_i-B_i),
&
\mathrm{TLG} &= \frac{1}{T}\sum_{i=1}^{T}(A_i-P_i),\\
\mathrm{BWT} &= \frac{1}{T-1}\sum_{i=1}^{T-1}(F_i-A_i).
\end{aligned}
\end{equation}
FWT and BWT follow \citet{gem2017}, while TLG is our addition. FWT and BWT
average over $T{-}1$ applicable tasks because forward transfer is undefined for
the first task and backward transfer is undefined for the last task. TLG
averages direct learning over all $T$ tasks. Means across the five settings
weight each setting equally.

\subsection{Representative Task Examples}
\label{app:task_examples}

Figure~\ref{fig:example_grid} shows one task from each setting.
\begin{figure}[H]
  \centering
  \includegraphics[width=\linewidth]{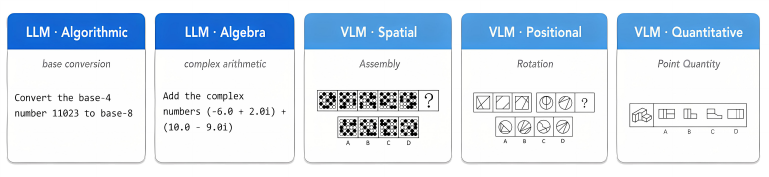}
  \caption{\textbf{Representative CRG tasks.} One example is shown for each
  setting.}
  \label{fig:example_grid}
\end{figure}

The transcripts below show one evaluation prompt and a model response in the
format consumed by the corresponding verifier. They include the model's
\texttt{<think>} reasoning and \texttt{<answer>} span. The latter supplies the
answer scored by the task verifier.

\subsubsection{LLM Reasoning}

\casetitle{Base Conversion}{LLM algorithmic}
\begin{lstlisting}[style=caseTranscript]
[PROMPT]
Your task is to convert a number between two different bases.

If the target base is > 10, use lowercase letters a-z for digits above 9.

Now, convert the base-4 number 11023 to base-8

[MODEL RESPONSE]
<think>
Okay, let's see. I need to convert the base-4 number 11023 to base-8. Hmm, how do I do that? Well, I remember that converting between bases can be done by first converting the number to decimal (base 10) and then to the target base. Let me try that.

First, let's convert 11023 (base 4) to decimal. Each digit represents a power of 4. Starting from the right, the digits are positions 0 to 4. So the number is 1 1 0 2 3.

So, the digits are 1, 1, 0, 2, 3. Let me calculate each digit's value:

1 * 4^4 = 1 * 256 = 256
1 * 4^3 = 1 * 64 = 64
0 * 4^2 = 0 * 16 = 0
2 * 4^1 = 2 * 4 = 8
3 * 4^0 = 3 * 1 = 3

Now add them all up: 256 + 64 is 320, plus 0 is 320, plus 8 is 328, plus 3 is 331. So the decimal equivalent is 331.

Now, I need to convert 331 (base 10) to base 8. To do that, I divide the number by 8 and keep track of the remainders.

331 divided by 8 is 41 with a remainder of 3.
41 divided by 8 is 5 with a remainder of 1.
5 divided by 8 is 0 with a remainder of 5.

So the remainders, from last to first, are 5, 1, 3. Therefore, the base 8 number is 513.

Let me check that. 5*8^2 + 1*8 + 3 = 5*64 + 8 +3 = 320 + 11 = 331. Yes, that's correct. So the answer should be 513.
</think>

<answer>513</answer>
\end{lstlisting}

\medskip

\casetitle{Complex Arithmetic}{LLM algebra}
\begin{lstlisting}[style=caseTranscript]
[PROMPT]
Add the complex numbers: (-6.0 + 2.0i) + (10.0 - 9.0i)

[MODEL RESPONSE]
<think>
Okay, let's see. I need to add two complex numbers: (-6.0 + 2.0i) and (10.0 - 9.0i). Hmm, how do you add complex numbers again? Oh right, you just add the real parts together and the imaginary parts together.

So the first number is -6.0 plus 2.0i. The second one is 10.0 minus 9.0i. Let me break it down. The real parts are -6.0 and 10.0. Adding those together: -6.0 + 10.0. That should be 4.0, right? Because 10 minus 6 is 4.

Now the imaginary parts: 2.0i and -9.0i. Adding those gives 2.0 - 9.0, which is -7.0i. So combining those two results, the sum should be 4.0 - 7.0i. Let me check again to make sure I didn't mix up any signs. The second number's imaginary part is negative, so subtracting 9.0i from 2.0i would be -7.0i. Yeah, that seems right. So the final answer is 4.0 - 7.0i.
</think>

<answer>4.0 - 7.0i</answer>
\end{lstlisting}

\subsubsection{VLM Visual Reasoning}
Each visual example is a four-way VisuLogic-derived problem whose image contains
the answer options. We report the image, shared textual prompt, and model
response.

\casetitle{Linear Quantity}{VLM quantitative}
\begin{center}\includegraphics[width=0.5\linewidth]{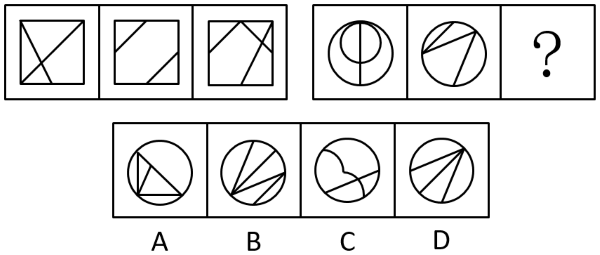}\end{center}
\begin{lstlisting}[style=caseTranscript]
[PROMPT]
From the four given options, select the most suitable one to fill in the question mark, so that a certain regularity is presented:

A: A  
B: B  
C: C  
D: D
[four answer options A--D are shown in the image above]

[MODEL RESPONSE]
<think>The pattern involves dividing shapes into sections and rotating these sections within the shape.

1. In the first row, the square is divided diagonally into two triangles, which are then rotated and shifted within the square.
2. Applying this logic to the circles:
   - The circle is divided into three equal parts (sectors).
   - These sectors are then rotated and shifted within the circle to form the pattern seen in the last column of the first row.

The correct option should reflect this rotation and shifting pattern within the circle. Option A matches this description as it shows the sectors rotated and shifted similarly to the pattern observed in the first row's last column.</think>
<answer>A</answer>
\end{lstlisting}

\medskip

\casetitle{Polyhedron Rotation}{VLM spatial}
\begin{center}\includegraphics[width=0.5\linewidth]{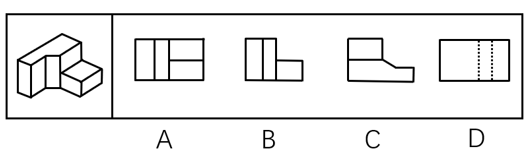}\end{center}
\begin{lstlisting}[style=caseTranscript]
[PROMPT]
From any angle, which one on the right is not a view of the three-dimensional shape given on the left?

A: A
B: B
C: C
D: D
[four answer options A--D are shown in the image above]

[MODEL RESPONSE]
<think>Observation:

1. Shape left:
   - Block structure vertically stacked horizontally.
   - Visible rectangular planes intersecting.

2. Options:
   - A: Horizontal split rectangles aligned vertically.
   - B: Vertical split rectangles aligned horizontally.
   - C: Horizontal split rectangles aligned horizontally.
   - D: Rectangular dashed lines vertically aligned.

3. Analysis:
   - Option A matches horizontal splits vertically aligned.
   - Option B matches vertical splits horizontally aligned.
   - Option C matches horizontal splits horizontally aligned.
   - Option D dashed vertical rectangles align incorrectly.</think>
<answer>D</answer>
\end{lstlisting}

\medskip

\casetitle{Rotation}{VLM positional}
\begin{center}\includegraphics[width=0.5\linewidth]{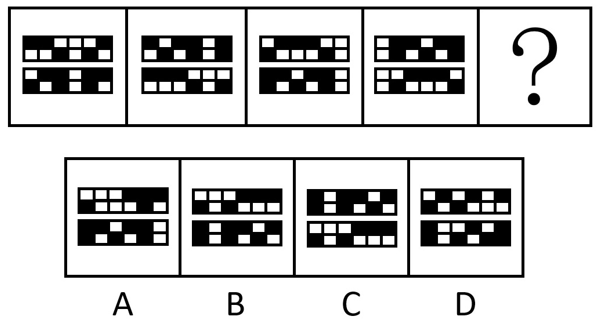}\end{center}
\begin{lstlisting}[style=caseTranscript]
[PROMPT]
From the four given options, choose the most appropriate one to fill in the question mark, so that a certain pattern is formed:

A: A
B: B
C: C
D: D
[four answer options A--D are shown in the image above]

[MODEL RESPONSE]
<think>
The arrangement of squares within each grid seems to alternate between two distinct patterns:

1. Grids with horizontal alignment.
2. Grids with vertical alignment.

Observing closely:
- Option A mirrors horizontally aligned squares.
- Option B mirrors vertically aligned squares.
- Option C mirrors horizontally aligned squares.
- Option D mirrors vertically aligned squares.

Given the alternating pattern:
Option B matches vertically aligned squares.

</think>
<answer>B</answer>
\end{lstlisting}

\clearpage
\raggedbottom
\section{Experimental Details}

\subsection{Training Configuration}\label{app:hyperparameters}

\begin{table}[H]
  \centering
  \small
  \caption{\textbf{All runs use the same core RLVR configuration.} The text
  experiments use Qwen3-4B, and the visual experiments use Qwen2.5-VL-7B.
  Training updates all model parameters, including the VLM vision tower.}
  \label{tab:hyperparameters}
  \begin{tabularx}{\linewidth}{@{}p{0.18\linewidth}p{0.32\linewidth}X@{}}
    \toprule
    Group & Hyperparameter & Value \\
    \midrule
    Model & LLM backbone & Qwen3-4B \\
          & VLM backbone & Qwen2.5-VL-7B \\
    \midrule
    RLVR & Advantage estimator & GRPO \\
         & Entropy coefficient & $1\times10^{-3}$ \\
         & Rollouts per prompt & 5 \\
         & PPO clip range & $0.2$ \\
         & PPO epochs & 1 \\
         & Loss reduction & Token mean \\
    \midrule
    Optimization & Optimizer & AdamW \\
                 & Learning rate & $1\times10^{-6}$ \\
                 & Adam $\beta$ values & $(0.9,0.999)$ \\
                 & Weight decay & $0.01$ \\
                 & Learning-rate schedule & Constant \\
                 & Warmup steps & 0 \\
                 & Gradient clipping & $1.0$ \\
                 & Train batch size & 512 \\
                 & PPO mini-batch size & 256 \\
                 & PPO micro-batch per GPU & 10 for LLM, 4 for VLM \\
                 & Total optimizer steps & 500 \\
    \midrule
    Generation & Temperature & $1.0$ \\
               & Top-$p$ & $1.0$ \\
               & Top-$k$ & $-1$ \\
    \midrule
    Sequence length & Maximum prompt length for LLM & 1024 \\
                    & Maximum response length for LLM & 1024 \\
                    & Maximum prompt length for VLM & 1024 \\
                    & Maximum response length for VLM & 1024 \\
                    & Maximum model length for VLM & 3072 \\
    \midrule
    Parameter update & LoRA rank & 0, corresponding to full-parameter updates \\
                     & VLM vision tower & Unfrozen \\
    \bottomrule
  \end{tabularx}
\end{table}

\begin{table}[H]
  \centering
  \small
  \caption{\textbf{MTRL accesses the complete task pool from the first update,
  whereas CPR replays previous-task prompts.} Both use the shared configuration
  in Table~\ref{tab:hyperparameters}.}
  \label{tab:mtrl-cpr-configuration}
  \begin{tabularx}{\linewidth}{@{}p{0.13\linewidth}p{0.31\linewidth}X@{}}
    \toprule
    Method & Setting & Value \\
    \midrule
    MTRL & Task access & Complete task pool available from the first update \\
         & LLM task sampling & $P_{\mathrm{MTRL}}(k)=1/T$ \\
    \midrule
    CPR & Replay pool & Prompts from previous tasks \\
        & Replay fraction & $\rho=0.5$ \\
        & Maximum replayed prompts per batch & $\min\{\lfloor\rho b\rfloor,\lfloor b/2\rfloor\}$ for batch size $b$ \\
        & Store capacity & One record per observed prompt with task identity and verifier metadata. Footprint reported in Appendix~\ref{app:cpr-resource-overhead} \\
        & Eligibility & Latest pass rate in $[0.24,0.70]$ \\
        & Priority & Smallest distance from pass rate $0.5$ first \\
        & Cooldown & 5 updates \\
    \bottomrule
  \end{tabularx}
\end{table}

\begin{table}[H]
  \centering
  \small
  \caption{\textbf{Each CL baseline changes a specific part of
  the shared training protocol.} The table lists every baseline-specific
  override. All remaining settings follow Table~\ref{tab:hyperparameters}.}
  \label{tab:baseline-configuration}
  \begin{tabularx}{\linewidth}{@{}p{0.19\linewidth}p{0.31\linewidth}X@{}}
    \toprule
    Method & Setting & Value \\
    \midrule
    EWC & Penalty coefficient & $1.0$ \\
        & Fisher decay & $1.0$ \\
        & Numerical stabilizer $\epsilon$ & $10^{-8}$ \\
        & Fisher samples & All available samples \\
    \midrule
    FIRE & Newton--Schulz steps & 5 \\
         & Reset modules & Attention query and key projections \\
         & Task-boundary handling & Reset the optimizer and synchronize rollout weights after reset \\
    \midrule
    ReDo & Dormancy threshold & $\tau=0.1$ \\
         & Calibration set & 8 examples \\
         & LeCun initialization & Disabled \\
    \midrule
    Muon & Momentum & $0.95$ \\
         & Newton--Schulz steps & 5 \\
         & Nesterov momentum & Enabled \\
         & AdamW-side $\beta$ values & $(0.9,0.999)$ \\
         & AdamW-side $\epsilon$ & $10^{-8}$ \\
         & Learning-rate scaling & RMS-matched AdamW \\
         & Weight-decay scaling & $1.0$ \\
         & Maximum matrix aspect ratio & 8 \\
         & Maximum matrix dimension & $65{,}536$ \\
    \midrule
    OSFT & Rank ratio & $0.5$ \\
         & Initialization & At the start of training and after every task switch \\
         & Task-boundary handling & Reset the optimizer and synchronize rollout weights after reinitialization \\
         & Numerical precision & FP32 upcast with BF16 output \\
    \midrule
    KL-to-old-policy & KL coefficient & $0.001$ \\
                     & Old-prompt fraction & $0.5$ \\
                     & Old-policy refresh & At every task switch \\
                     & Checkpoint requirement & Task-boundary checkpoint \\
    \bottomrule
  \end{tabularx}
\end{table}

\clearpage
\subsection{CPR Implementation}
\label{app:cpr-details}

\begin{figure}[H]
  \centering
  \includegraphics[width=\linewidth]{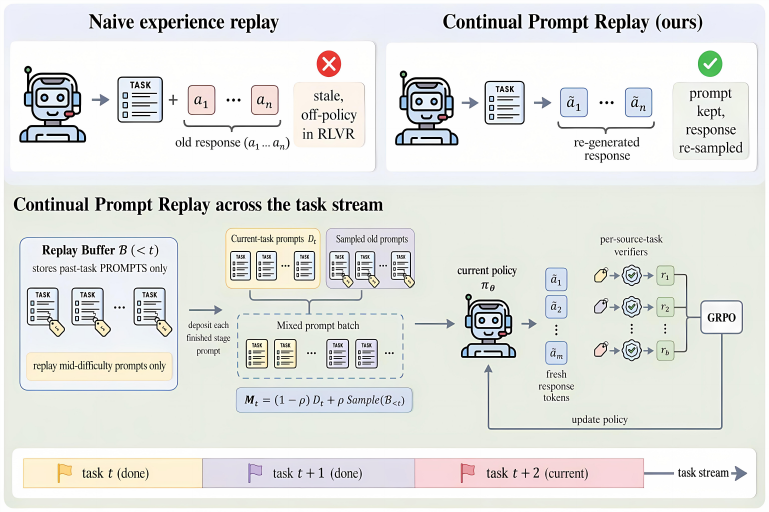}
  \caption{\textbf{CPR turns previous-task prompts into current-policy RLVR
  updates.} The top row contrasts stale trajectory replay with CPR's response
  regeneration. The lower panel shows one CPR update during a later stage. CPR selects
  previous-task prompts, mixes them with current-task prompts, generates every
  response with the current policy, and scores each response with its task
  verifier before the GRPO update.}
  \label{fig:cpr_method}
\end{figure}

\clearpage
Algorithm~\ref{alg:continual_prompt_replay} specifies one CPR update during
stage $t$.

\begin{center}
  \refstepcounter{algorithm}\label{alg:continual_prompt_replay}
  \setlength{\fboxsep}{6pt}
  \fbox{\begin{minipage}{0.95\textwidth}
  \small
  \textbf{Algorithm~\thealgorithm: One CPR update during stage $t$}\\[2pt]
  \begin{tabularx}{\textwidth}{@{}rX@{}}
  \textbf{Input:} & current-task prompt distribution $\mathcal{D}_t$, replay store $\mathcal{B}_{<t}$, task verifiers $\{r_k\}_{k\le t}$, current policy $\pi_\theta$, batch size $b$, and replay fraction $\rho$. \\
  1 & set the maximum number of replayed prompts to $b_R=\min\{\lfloor\rho b\rfloor,\lfloor b/2\rfloor\}$. \\
  2 & construct a batch of $b$ prompts by replacing at most $b_R$ current-task prompts with eligible prompts from $\mathcal{B}_{<t}$. Let $C$ and $R$ denote the resulting current-task and replay subsets, and set $\widehat\rho_t=|R|/b$ for this update. \\
  3 & form the mixed prompt batch $C\cup R$ with each prompt's task identity $m(x)$. \\
  4 & generate a fresh response $y\sim\pi_\theta(\cdot\mid x)$ for every prompt $x\in C\cup R$. \\
  5 & score each response with its task verifier, $r_{m(x)}(x,y)$. \\
  6 & apply one RLVR update to the mixed batch. \\
  7 & refresh the stored record and latest pass rate for each observed prompt, retaining at most one record per prompt. \\
  \textbf{Output:} & updated policy $\pi_{\theta'}$ and replay store $\mathcal{B}_{\le t}$.
  \end{tabularx}
  \end{minipage}}
\end{center}

\paragraph{Replay store and selector.}
CPR stores one record per prompt and refreshes its pass rate whenever the prompt
reappears. The pass rate is the mean verifier score over the five responses
generated for that prompt in its latest update.

The default selector admits previous-task prompts whose latest pass rate lies in
$[0.24,0.70]$. A selected prompt becomes eligible again after five updates.
Among eligible prompts, CPR prioritizes those closest to pass rate $0.5$.

For one update during stage $t$, let $n_{t,k}$ be the number of replayed prompts
drawn from task $k<t$, and let $n_t=\sum_{\ell<t}n_{t,\ell}$. When $n_t>0$,
$q_{t,k}=n_{t,k}/n_t$ and $\sum_{k<t}q_{t,k}=1$, so $q_{t,k}$ is the fraction of
replayed prompts in that update drawn from task $k$. When $n_t=0$,
$\widehat\rho_t=0$ and the replay term is omitted. If fewer than $b_R$ prompts
are eligible, then $0<\widehat\rho_t<\rho$ whenever at least one prompt is
replayed.

\subsection{CPR Resource Overhead}
\label{app:cpr-resource-overhead}

At each CPR update, previous-task prompts replace current-task prompts within
the same 512-prompt batch used by Seq. RLVR, and the current policy draws five
responses per prompt. Both methods therefore generate 2,560 rollouts per
update. CPR changes the task-sampling distribution without increasing the
number of sampled prompts or current-policy rollouts. Its method-specific work
is replay selection and CPU-side store access.

\paragraph{Replay-store footprint.}
CPR stores one record per observed prompt containing the prompt, task and
verifier metadata, latest pass rate, and selector state. The serialized
footprint is 1.24~KiB per prompt on average across the five CRG settings
(Table~\ref{tab:cpr-store-footprint}).

\begin{table}[H]
  \centering
  \small
  \setlength{\tabcolsep}{12pt}
  \caption{\textbf{Average serialized footprint per CPR replay record.}
  Each row reports the mean over prompts in one CRG setting. The final row
  weights the five settings equally.}
  \label{tab:cpr-store-footprint}
  \begin{tabular}{lr}
    \toprule
    Setting & Per prompt (KiB) \\
    \midrule
    LLM Algorithmic   & 1.02 \\
    LLM Algebra       & 1.01 \\
    VLM Quantitative  & 1.42 \\
    VLM Spatial       & 1.39 \\
    VLM Positional    & 1.34 \\
    \midrule
    Mean              & 1.24 \\
    \bottomrule
  \end{tabular}
\end{table}

\clearpage
\paragraph{Runtime.}
On eight NVIDIA B200 GPUs, CPR averages 14.8 hours per 500-update run, compared
with 14.6 hours for Seq.\ RLVR (Table~\ref{tab:cpr-runtime}).

\begin{table}[H]
  \centering
  \small
  \setlength{\tabcolsep}{12pt}
  \caption{\textbf{Average wall-clock time per 500-update run on eight
  NVIDIA B200 GPUs.}}
  \label{tab:cpr-runtime}
  \begin{tabular}{lr}
    \toprule
    Method & Mean runtime (h) \\
    \midrule
    Seq.\ RLVR & 14.6 \\
    CPR          & 14.8 \\
    \bottomrule
  \end{tabular}
\end{table}

\FloatBarrier
\flushbottom
\section{Theoretical Analysis}

\subsection{Proof of Proposition~\ref{prop:decomposition}}\label{app:decomposition-proof}
Adding and subtracting $P_i$ and $A_i$ writes each per-task final score as
\begin{equation}\label{eq:proof-Fi-split}
F_i=B_i+(P_i-B_i)+(A_i-P_i)+(F_i-A_i),
\end{equation}
and averaging over the $T$ tasks gives
\begin{equation}\label{eq:proof-averaged}
\begin{aligned}
\mathrm{FinalAvg}
&=\mathrm{BaseAvg}
+\frac{1}{T}\sum_{i=1}^{T}(P_i-B_i)
+\mathrm{TLG}
+\frac{1}{T}\sum_{i=1}^{T}(F_i-A_i)\\
&=\mathrm{BaseAvg}
+\frac{1}{T}\sum_{i=2}^{T}(P_i-B_i)
+\mathrm{TLG}
+\frac{1}{T}\sum_{i=1}^{T-1}(F_i-A_i),
\end{aligned}
\end{equation}
where the boundary terms $P_1-B_1$ and $F_T-A_T$ vanish by the conventions of
Appendix~\ref{app:metric-definitions}. Matching the $\tfrac{1}{T}$ averaging to the
$\tfrac{1}{T-1}$ normalization of the transfer terms,
\begin{equation}\label{eq:proof-normalize}
\frac{1}{T}\sum_{i=2}^{T}(P_i-B_i)=\frac{T-1}{T}\,\mathrm{FWT},
\qquad
\frac{1}{T}\sum_{i=1}^{T-1}(F_i-A_i)=\frac{T-1}{T}\,\mathrm{BWT},
\end{equation}
which yields the identity of Proposition~\ref{prop:decomposition}.

\clearpage
\subsection{First-Order CPR Analysis}
\label{app:prompt_replay_details}

By Proposition~\ref{prop:decomposition}, any FinalAvg difference between
methods with the same BaseAvg must pass through FWT, TLG, or BWT. We therefore
compare one CPR update with a compute-matched baseline update and ask when
replay changes these terms favorably.

Consider one update in stage $t$ at pre-update parameters $\theta$. Let
$g_i^{(t)}=\nabla_\theta J_i(\theta)$ be the objective gradient for task $i$, and
let $g_k^{(t)}$ be the update direction induced by prompts from task $k$.
CPR combines prompts from the current task $t$ with prompts from
previous tasks $k<t$ through the direction $g_t^{\mathrm{CPR}}$ in
Eq.~\eqref{eq:cpr-mixed-gradient}.
A first-order expansion of task $i$ under CPR gives
\begin{equation}\label{eq:taylor-general}
\begin{aligned}
  \Delta_{i,\mathrm{CPR}}^{(t)}
  &:=J_i(\theta+\eta g_t^{\mathrm{CPR}})-J_i(\theta)\\
  &=\eta(1-\widehat\rho_t)\langle g_i^{(t)},g_t^{(t)}\rangle
  +\eta\widehat\rho_t\sum_{k<t}q_{t,k}
  \langle g_i^{(t)},g_k^{(t)}\rangle+O(\eta^2).
\end{aligned}
\end{equation}
The Seq.\ RLVR baseline uses only the current-task direction, so
\begin{equation}\label{eq:taylor-sequential}
  \Delta_{i,\mathrm{seq}}^{(t)}
  =\eta\langle g_i^{(t)},g_t^{(t)}\rangle+O(\eta^2).
\end{equation}
Subtracting Eq.~\eqref{eq:taylor-sequential} from
Eq.~\eqref{eq:taylor-general} yields the central comparison:
\begin{equation}\label{eq:cpr-vs-sequential}
\begin{aligned}
  \Delta_{i,\mathrm{CPR}}^{(t)}-\Delta_{i,\mathrm{seq}}^{(t)}
  =\eta\widehat\rho_t\left[
  \sum_{k<t}q_{t,k}\langle g_i^{(t)},g_k^{(t)}\rangle
  -\langle g_i^{(t)},g_t^{(t)}\rangle\right]+O(\eta^2).
\end{aligned}
\end{equation}
The bracket compares the average direction recovered from previous-task prompts
with the current-task direction that those prompts replace. CPR improves task
$i$ to first order when this difference is positive after the relevant updates
are averaged.

If every task pair had non-negative alignment, the current-task direction would
already be non-harmful to every objective at first order. Mean alignment is
weaker. The current direction may interfere with specific objectives even when
the average over previous-task directions and relevant objectives is positive.
CPR realizes that average at the current parameters by regenerating responses to
previous-task prompts. Assumption~\ref{asm:shared-reasoning} gives a
non-negative expected first-order contribution under a reported term's
weighting. Equation~\eqref{eq:cpr-vs-sequential} gives the corresponding
CPR--Seq.\ RLVR difference.

The position of task $i$ relative to stage $t$ determines which decomposition
term receives the local effect. Setting $i=t$ gives the contribution during
direct training and hence TLG. Tasks with $i<t$ give the contribution to BWT,
while tasks with $i>t$ give the contribution to FWT. The replay pool controls
the task mixture $q_{t,k}$. The replay ratio controls the scale
$\widehat\rho_t$. The replay-pool ablation probes $q_{t,k}$. The ratio sweep
in Appendix~\ref{app:cpr-ratio-robustness} tests robustness to the configured
replay scale $\rho$.

\clearpage
\subsection{Corollary Proofs}
\label{app:corollary-proofs}

We first expand the expectation in Assumption~\ref{asm:shared-reasoning}.
For one reported term, let $\alpha_t$ denote its non-negative stage weights,
$p_{t,i}$ its non-negative task-objective weights, and $w_{t,k}$ the realized
sampling frequencies of prompts from task $k$. Each set of weights is normalized
over its support. For an update mixture with these weights, define
$\Delta_i^{(t)}=J_i(\theta+\eta\sum_k w_{t,k}g_k^{(t)})-J_i(\theta)$. Then
\begin{equation}\label{eq:mean-gradient-alignment-expanded}
\mathbb E_{t,i,k}\!\left[
\langle g_i^{(t)},g_k^{(t)}\rangle\right]
\defeq \sum_t\alpha_t\sum_i p_{t,i}\sum_k w_{t,k}
\mathbb E_{\mathrm{upd}}\!\left[
\langle g_i^{(t)},g_k^{(t)}\rangle\right],
\end{equation}
where $\mathbb E_{\mathrm{upd}}$ averages prompt, selector, rollout, and update
stochasticity. Averaging the corresponding first-order task changes gives
\begin{equation}\label{eq:taylor-aggregate}
\sum_t\alpha_t\sum_i p_{t,i}
\mathbb E_{\mathrm{upd}}\!\left[\Delta_i^{(t)}\right]
=\eta\sum_t\alpha_t\sum_i p_{t,i}\sum_k w_{t,k}
\mathbb E_{\mathrm{upd}}\!\left[
\langle g_i^{(t)},g_k^{(t)}\rangle\right]+O(\eta^2).
\end{equation}

For Corollary~\ref{cor:low-forgetting}, the Seq.\ RLVR baseline has $w_{t,t}=1$.
BWT uses later stages and earlier-task objectives, so its first-order
contribution is proportional to
\begin{equation}\label{eq:bwt-mean-alignment}
\sum_t\alpha_t\sum_{i<t}p_{t,i}
\mathbb E_{\mathrm{upd}}\!\left[
\langle g_i^{(t)},g_t^{(t)}\rangle\right].
\end{equation}
This is Assumption~\ref{asm:shared-reasoning} under BWT's stage and task
weighting with $k=t$. Its non-negativity proves a non-negative expected
first-order contribution to earlier-task objectives even when individual
$(i,t)$ pairs are negative.

For Corollary~\ref{cor:joint-improvement}, CPR assigns
$w_{t,t}=1-\widehat\rho_t$ and
$w_{t,k}=\widehat\rho_t q_{t,k}$ for $k<t$. Applying
Eq.~\eqref{eq:taylor-aggregate} with $i=t$ gives the first-order contribution to
TLG. Restricting the task weights to $i<t$ gives BWT, while restricting them to
$i>t$ gives FWT. Whenever Assumption~\ref{asm:shared-reasoning} holds under one
of these weightings, the corresponding right-hand side is non-negative. This
proves the corollary for that term. Equation~\eqref{eq:cpr-vs-sequential}
gives the corresponding CPR--Seq.\ RLVR difference.

\clearpage
\section{Additional Results}

\subsection{Full Method Results}

Method rankings vary substantially across task sequences. Figure~\ref{fig:cpr} compares CTM
for all eight sequential methods. Figure~\ref{fig:decomposition} compares the
exact FinalAvg decompositions of the Seq.\ RLVR baseline versus CPR for each task sequence.
Tables~\ref{tab:formal45-llm-algorithmic}--\ref{tab:formal45-vlm-positional-reasoning}
report the complete results.

\begin{figure}[H]
  \centering
  \includegraphics[width=0.95\linewidth]{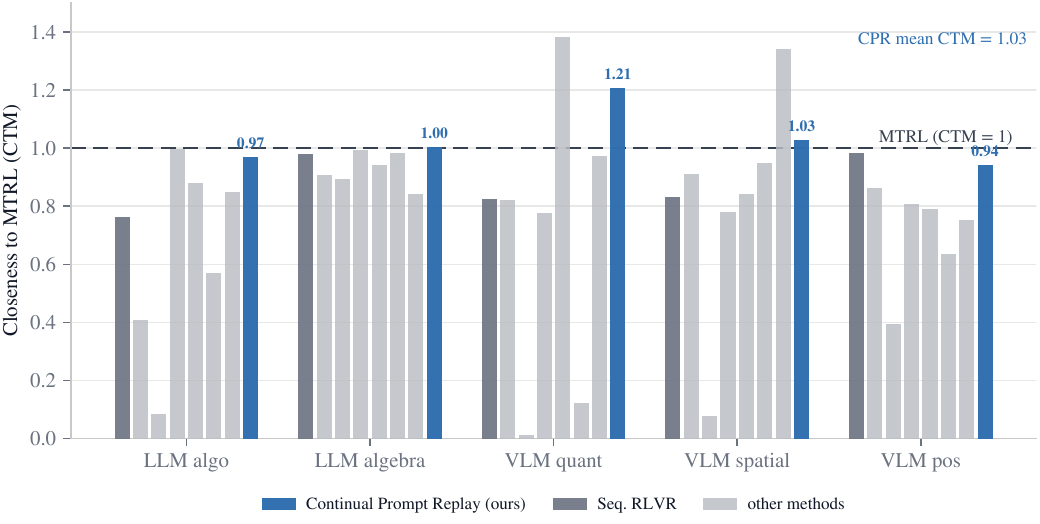}
  \caption{\textbf{CPR is the only evaluated CL method that reaches
  MTRL-level performance on average.} Bars show CTM for eight sequential
  methods in each setting. Seq.\ RLVR is dark gray. The six other interventions
  are light gray. CPR is blue. This level is reached or
  exceeded for LLM algebra, VLM quantitative reasoning, and VLM spatial
  reasoning. Its mean CTM is $1.03$, compared with $0.88$ for
  Seq.\ RLVR.}
  \label{fig:cpr}
\end{figure}

The same FinalAvg difference can reflect different changes in FWT, TLG, and BWT.
Figure~\ref{fig:decomposition} compares the exact FinalAvg decomposition for
the Seq.\ RLVR baseline and CPR for each task sequence.

\begin{figure}[H]
  \centering
  \includegraphics[width=0.96\linewidth]{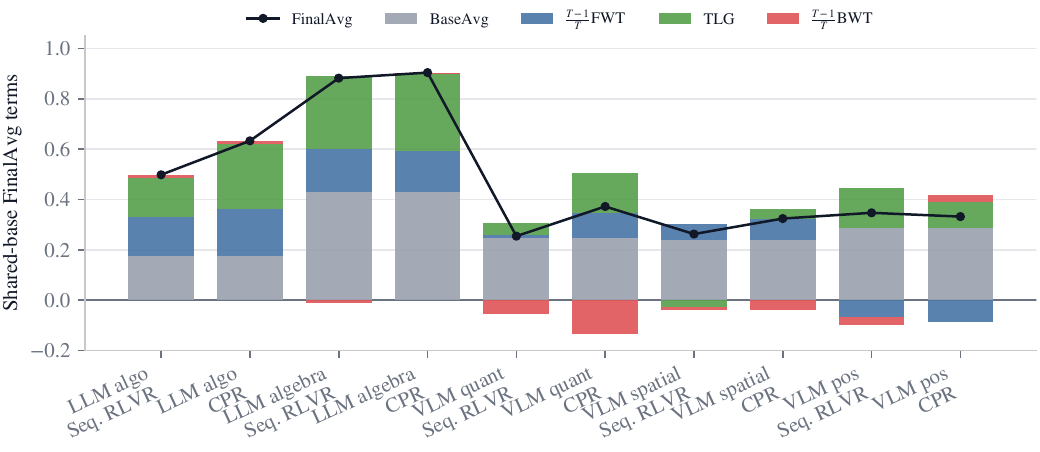}
  \caption{\textbf{CPR's FinalAvg gains arise from FWT, TLG, or both.}
  Each pair compares Seq.\ RLVR with CPR. Stacked bars show
  the four terms BaseAvg, $\frac{T-1}{T}\mathrm{FWT}$, TLG,
  $\frac{T-1}{T}\mathrm{BWT}$. Black markers show FinalAvg. The BWT effect varies
  across task sequences. On VLM quantitative reasoning, larger FWT and TLG
  outweigh a more negative BWT contribution.}
  \label{fig:decomposition}
\end{figure}

\begin{table}[H]
\centering
\small
\caption{\textbf{Complete results for LLM algorithmic reasoning.} Base reports pre-training accuracy. Final reports endpoint accuracy. Final$-$Base, FWT, TLG, BWT use percentage points. The largest reported value in each performance column is bold.}
\label{tab:formal45-llm-algorithmic}
\setlength{\tabcolsep}{2pt}
\renewcommand{\arraystretch}{1.08}
\resizebox{\linewidth}{!}{%
\begin{tabular}{p{0.70in} p{1.08in} p{0.90in} p{0.34in} p{0.44in} p{0.50in} p{0.60in} p{0.50in} p{0.50in} p{0.50in}}
\toprule
Role & Method & Model & Steps & Base (\%) & Final (\%) & Final$-$Base (\%) & FWT (\%) & TLG (\%) & BWT (\%) \\
\midrule
baseline & Seq.\ RLVR & Qwen3-4B & 500 & 17.5 & 49.9 & +32.3 & +17.2 & +15.4 & \textbf{+1.6} \\
MTRL reference & MTRL & Qwen3-4B & 500 & 17.5 & \textbf{65.3} & \textbf{+47.7} & -- & -- & -- \\
CL method & FIRE & Qwen3-4B & 500 & 17.5 & 5.6 & -11.9 & +10.0 & +12.9 & -37.5 \\
CL method & ReDo & Qwen3-4B & 500 & 17.5 & 26.8 & +9.2 & +9.5 & +4.6 & -4.4 \\
CL method & EWC & Qwen3-4B & 500 & 17.5 & 65.0 & +47.5 & +19.9 & \textbf{+30.1} & -0.6 \\
CL method & Muon & Qwen3-4B & 500 & 17.5 & 57.4 & +39.9 & +20.0 & +23.7 & -2.0 \\
CL method & OSFT & Qwen3-4B & 500 & 17.5 & 37.3 & +19.8 & +12.5 & +18.9 & -11.5 \\
CL method & CPR & Qwen3-4B & 500 & 17.5 & 63.3 & +45.8 & \textbf{+20.7} & +26.0 & +1.2 \\
CL method & KL-to-old-policy & Qwen3-4B & 500 & 17.5 & 55.4 & +37.8 & +20.1 & +23.4 & -4.0 \\
\bottomrule
\end{tabular}%
}
\end{table}
\begin{table}[H]
\centering
\small
\caption{\textbf{Complete results for LLM algebra.} Base reports pre-training accuracy. Final reports endpoint accuracy. Final$-$Base, FWT, TLG, BWT use percentage points. The largest reported value in each performance column is bold.}
\label{tab:formal45-llm-algebra}
\setlength{\tabcolsep}{2pt}
\renewcommand{\arraystretch}{1.08}
\resizebox{\linewidth}{!}{%
\begin{tabular}{p{0.70in} p{1.08in} p{0.90in} p{0.34in} p{0.44in} p{0.50in} p{0.60in} p{0.50in} p{0.50in} p{0.50in}}
\toprule
Role & Method & Model & Steps & Base (\%) & Final (\%) & Final$-$Base (\%) & FWT (\%) & TLG (\%) & BWT (\%) \\
\midrule
baseline & Seq.\ RLVR & Qwen3-4B & 500 & 43.2 & 88.2 & +45.1 & +20.3 & +29.1 & -1.1 \\
MTRL reference & MTRL & Qwen3-4B & 500 & 43.2 & 90.0 & +46.8 & -- & -- & -- \\
CL method & CPR & Qwen3-4B & 500 & 43.2 & \textbf{90.4} & \textbf{+47.2} & +19.6 & +30.9 & +0.1 \\
CL method & OSFT & Qwen3-4B & 500 & 43.2 & 88.4 & +45.2 & +23.6 & +21.9 & +4.3 \\
CL method & FIRE & Qwen3-4B & 500 & 43.2 & 80.3 & +37.1 & +15.0 & \textbf{+33.0} & -10.0 \\
CL method & ReDo & Qwen3-4B & 500 & 43.2 & 81.7 & +38.5 & +19.1 & +23.2 & -0.6 \\
CL method & EWC & Qwen3-4B & 500 & 43.2 & 89.3 & +46.2 & +19.7 & +30.9 & -1.3 \\
CL method & KL-to-old-policy & Qwen3-4B & 500 & 43.2 & 75.7 & +32.6 & +16.1 & +4.3 & \textbf{+17.7} \\
CL method & Muon & Qwen3-4B & 500 & 43.2 & 84.8 & +41.6 & \textbf{+25.0} & +23.4 & -3.2 \\
\bottomrule
\end{tabular}%
}
\end{table}
\clearpage
\begin{table}[H]
\centering
\small
\caption{\textbf{Complete results for VLM quantitative reasoning.} Base reports pre-training accuracy. Final reports endpoint accuracy. Final$-$Base, FWT, TLG, BWT use percentage points. The largest reported value in each performance column is bold.}
\label{tab:formal45-vlm-quantitative-reasoning}
\setlength{\tabcolsep}{2pt}
\renewcommand{\arraystretch}{1.08}
\resizebox{\linewidth}{!}{%
\begin{tabular}{p{0.70in} p{1.08in} p{0.90in} p{0.34in} p{0.44in} p{0.50in} p{0.60in} p{0.50in} p{0.50in} p{0.50in}}
\toprule
Role & Method & Model & Steps & Base (\%) & Final (\%) & Final$-$Base (\%) & FWT (\%) & TLG (\%) & BWT (\%) \\
\midrule
baseline & Seq.\ RLVR & Qwen2.5-VL-7B & 500 & 24.8 & 25.5 & +0.7 & +1.4 & +5.1 & -7.2 \\
MTRL reference & MTRL & Qwen2.5-VL-7B & 500 & 24.8 & 30.9 & +6.1 & -- & -- & -- \\
CL method & CPR & Qwen2.5-VL-7B & 500 & 24.8 & 37.2 & +12.5 & +13.2 & \textbf{+16.0} & -18.0 \\
CL method & ReDo & Qwen2.5-VL-7B & 500 & 24.8 & 25.4 & +0.6 & -1.4 & -3.0 & +6.2 \\
CL method & FIRE & Qwen2.5-VL-7B & 500 & 24.8 & 0.4 & -24.4 & +0.1 & -3.2 & -28.3 \\
CL method & EWC & Qwen2.5-VL-7B & 500 & 24.8 & 24.0 & -0.8 & +2.9 & -10.7 & +10.3 \\
CL method & Muon & Qwen2.5-VL-7B & 500 & 24.8 & \textbf{42.7} & \textbf{+17.9} & \textbf{+13.9} & +6.3 & +1.6 \\
CL method & OSFT & Qwen2.5-VL-7B & 500 & 24.8 & 3.8 & -21.0 & -17.0 & -11.3 & +3.9 \\
CL method & KL-to-old-policy & Qwen2.5-VL-7B & 500 & 24.8 & 30.0 & +5.3 & +7.2 & -15.2 & \textbf{+20.0} \\
\bottomrule
\end{tabular}%
}
\end{table}
\begin{table}[H]
\centering
\small
\caption{\textbf{Complete results for VLM spatial reasoning.} Base reports pre-training accuracy. Final reports endpoint accuracy. Final$-$Base, FWT, TLG, BWT use percentage points. The largest reported value in each performance column is bold.}
\label{tab:formal45-vlm-spatial-reasoning}
\setlength{\tabcolsep}{2pt}
\renewcommand{\arraystretch}{1.08}
\resizebox{\linewidth}{!}{%
\begin{tabular}{p{0.70in} p{1.08in} p{0.90in} p{0.34in} p{0.44in} p{0.50in} p{0.60in} p{0.50in} p{0.50in} p{0.50in}}
\toprule
Role & Method & Model & Steps & Base (\%) & Final (\%) & Final$-$Base (\%) & FWT (\%) & TLG (\%) & BWT (\%) \\
\midrule
baseline & Seq.\ RLVR & Qwen2.5-VL-7B & 500 & 24.1 & 26.3 & +2.2 & +7.4 & -2.7 & -1.5 \\
MTRL reference & MTRL & Qwen2.5-VL-7B & 500 & 24.1 & 31.6 & +7.5 & -- & -- & -- \\
CL method & CPR & Qwen2.5-VL-7B & 500 & 24.1 & 32.5 & +8.4 & \textbf{+9.9} & +4.0 & -4.6 \\
CL method & ReDo & Qwen2.5-VL-7B & 500 & 24.1 & 28.7 & +4.7 & +1.9 & -3.7 & \textbf{+8.1} \\
CL method & FIRE & Qwen2.5-VL-7B & 500 & 24.1 & 2.5 & -21.6 & -1.7 & -6.5 & -16.4 \\
CL method & EWC & Qwen2.5-VL-7B & 500 & 24.1 & 24.7 & +0.6 & -7.7 & +8.7 & -2.0 \\
CL method & Muon & Qwen2.5-VL-7B & 500 & 24.1 & 26.7 & +2.6 & -4.2 & +1.4 & +5.7 \\
CL method & OSFT & Qwen2.5-VL-7B & 500 & 24.1 & 30.0 & +5.9 & +1.2 & -1.4 & +7.6 \\
CL method & KL-to-old-policy & Qwen2.5-VL-7B & 500 & 24.1 & \textbf{42.4} & \textbf{+18.4} & +7.7 & \textbf{+14.5} & -3.1 \\
\bottomrule
\end{tabular}%
}
\end{table}
\clearpage
\begin{table}[H]
\centering
\small
\caption{\textbf{Complete results for VLM positional reasoning.} Base reports pre-training accuracy. Final reports endpoint accuracy. Final$-$Base, FWT, TLG, BWT use percentage points. The largest reported value in each performance column is bold.}
\label{tab:formal45-vlm-positional-reasoning}
\setlength{\tabcolsep}{2pt}
\renewcommand{\arraystretch}{1.08}
\resizebox{\linewidth}{!}{%
\begin{tabular}{p{0.70in} p{1.08in} p{0.90in} p{0.34in} p{0.44in} p{0.50in} p{0.60in} p{0.50in} p{0.50in} p{0.50in}}
\toprule
Role & Method & Model & Steps & Base (\%) & Final (\%) & Final$-$Base (\%) & FWT (\%) & TLG (\%) & BWT (\%) \\
\midrule
baseline & Seq.\ RLVR & Qwen2.5-VL-7B & 500 & 28.7 & 34.7 & +6.0 & -8.9 & \textbf{+15.8} & -4.2 \\
MTRL reference & MTRL & Qwen2.5-VL-7B & 500 & 28.7 & \textbf{35.3} & \textbf{+6.6} & -- & -- & -- \\
CL method & CPR & Qwen2.5-VL-7B & 500 & 28.7 & 33.2 & +4.5 & -11.3 & +10.4 & \textbf{+3.4} \\
CL method & ReDo & Qwen2.5-VL-7B & 500 & 28.7 & 30.5 & +1.7 & \textbf{+0.7} & +2.3 & -1.6 \\
CL method & FIRE & Qwen2.5-VL-7B & 500 & 28.7 & 14.0 & -14.8 & -21.8 & +6.5 & -6.6 \\
CL method & Muon & Qwen2.5-VL-7B & 500 & 28.7 & 27.9 & -0.8 & -8.2 & +3.6 & +2.4 \\
CL method & EWC & Qwen2.5-VL-7B & 500 & 28.7 & 28.5 & -0.2 & -7.6 & +5.5 & +0.0 \\
CL method & OSFT & Qwen2.5-VL-7B & 500 & 28.7 & 22.5 & -6.2 & -4.0 & -5.2 & +2.6 \\
CL method & KL-to-old-policy & Qwen2.5-VL-7B & 500 & 28.7 & 26.5 & -2.2 & -9.6 & +10.8 & -7.8 \\
\bottomrule
\end{tabular}%
}
\end{table}

\clearpage
\subsection{Additional CPR Ablation Results}
\label{app:cpr-ablations}

We use two ablations to isolate CPR's replay mechanism.
Table~\ref{tab:cpr_resampling_ablation} tests current-policy regeneration.
Table~\ref{tab:cpr_source_ablation} changes the replay pool.

\begin{table}[H]
  \centering
  \small
  \setlength{\tabcolsep}{4.5pt}
  \caption{\textbf{Current-policy regeneration drives CPR's replay gain.}
  On LLM algorithmic reasoning at 500 steps and $\rho=0.5$, sample replay
  reuses earlier-policy trajectories with importance sampling. CPR regenerates
  responses with the current policy. Delta rows are relative to no replay and
  are computed from unrounded values. The best value among the three sequential
  variants in each metric is bold.}
  \label{tab:cpr_resampling_ablation}
  \begin{tabular}{lrrrr}
    \toprule
    Variant & FinalAvg (\%) & FWT (\%) & TLG (\%) & BWT (\%) \\
    \midrule
    No replay (Seq.\ RLVR) & 49.9 & +17.2 & +15.4 & +1.6 \\
    Sample replay & 47.5 & +19.6 & +10.3 & \textbf{+2.3} \\
    CPR & \textbf{63.3} & \textbf{+20.7} & \textbf{+26.0} & +1.2 \\
    MTRL reference & 65.3 & -- & -- & -- \\
    \midrule
    $\Delta$ Sample replay $-$ no replay & -2.4 & +2.4 & -5.1 & +0.6 \\
    $\Delta$ CPR $-$ no replay & +13.5 & +3.5 & +10.7 & -0.4 \\
    \bottomrule
  \end{tabular}
\end{table}

\clearpage
\begin{table}[H]
  \centering
  \scriptsize
  \setlength{\tabcolsep}{2.4pt}
  \caption{\textbf{Replaying prompts from all previous tasks gives the highest
  mean FinalAvg.} No replay corresponds to Seq.\ RLVR. Replay from only
  the immediately previous task is reported for LLM algorithmic reasoning.
  Bold values mark the maximum within each setting and among the mean rows.}
  \label{tab:cpr_source_ablation}
  \begin{tabular}{llrrrrr}
    \toprule
    Setting & Replay pool & FinalAvg (\%) & CTM & FWT (\%) & TLG (\%) & BWT (\%) \\
    \midrule
    LLM algorithmic & No replay                  & 49.9 & 0.76 & $+17.2$ & $+15.4$ & $\mathbf{+1.6}$ \\
    LLM algorithmic & Immediately previous task & \textbf{66.3} & \textbf{1.02} & $+20.8$ & $\mathbf{+30.1}$ & $+0.0$ \\
    LLM algorithmic & All previous tasks (CPR)  & 63.3 & 0.97 & $+20.7$ & $+26.0$ & $+1.2$ \\
    LLM algorithmic & Previous and current tasks & 64.3 & 0.98 & $\mathbf{+20.9}$ & $+29.1$ & $-1.3$ \\
    \midrule
    LLM algebra & No replay                  & 88.2 & 0.98 & $+20.3$ & $+29.1$ & $-1.1$ \\
    LLM algebra & All previous tasks (CPR)  & 90.4 & 1.00 & $+19.6$ & $+30.9$ & $\mathbf{+0.1}$ \\
    LLM algebra & Previous and current tasks & \textbf{91.2} & \textbf{1.01} & $\mathbf{+20.8}$ & $\mathbf{+31.3}$ & $-0.8$ \\
    \midrule
    VLM quantitative & No replay                  & 25.5 & 0.83 & $+1.4$  & $+5.1$  & $-7.2$ \\
    VLM quantitative & All previous tasks (CPR)  & \textbf{37.2} & \textbf{1.21} & $\mathbf{+13.2}$ & $\mathbf{+16.0}$ & $-18.0$ \\
    VLM quantitative & Previous and current tasks & 36.5 & 1.18 & $+10.9$ & $+3.4$  & $\mathbf{+0.1}$ \\
    \midrule
    VLM spatial & No replay                  & 26.3 & 0.83 & $+7.4$ & $-2.7$ & $-1.5$ \\
    VLM spatial & All previous tasks (CPR)  & \textbf{32.5} & \textbf{1.03} & $\mathbf{+9.9}$ & $\mathbf{+4.0}$ & $-4.6$ \\
    VLM spatial & Previous and current tasks & 29.8 & 0.94 & $+6.1$ & $-0.1$ & $\mathbf{+0.9}$ \\
    \midrule
    VLM positional & No replay                  & \textbf{34.7} & \textbf{0.98} & $-8.9$  & $\mathbf{+15.8}$ & $-4.2$ \\
    VLM positional & All previous tasks (CPR)  & 33.2 & 0.94 & $-11.3$ & $+10.4$ & $\mathbf{+3.4}$ \\
    VLM positional & Previous and current tasks & 31.1 & 0.88 & $\mathbf{-4.3}$  & $+6.7$  & $-1.4$ \\
    \midrule
    Mean & No replay                  & 44.9 & 0.88 & $+7.5$  & $+12.5$ & $-2.5$ \\
    Mean & All previous tasks (CPR)  & \textbf{51.3} & \textbf{1.03} & $+10.4$ & $\mathbf{+17.5}$ & $-3.6$ \\
    Mean & Previous and current tasks & 50.6 & 1.00 & $\mathbf{+10.9}$ & $+14.1$ & $\mathbf{-0.5}$ \\
    \bottomrule
  \end{tabular}
\end{table}

Replaying prompts from all previous tasks raises mean FinalAvg from $44.9\%$ to
$51.3\%$. Allowing current-task prompts reaches $50.6\%$.

\subsection{Robustness to Replay Ratio \texorpdfstring{$\rho$}{rho}}
\label{app:cpr-ratio-robustness}

CPR remains near MTRL across a tenfold range of nonzero replay ratios. Holding
all other settings fixed on LLM algorithmic reasoning, $\rho\in\{0.05,0.20,0.50\}$
gives CTM between $0.97$ and $0.99$.

\begin{table}[H]
  \centering
  \small
  \caption{\textbf{CPR is robust to replay ratio $\rho$ on LLM algorithmic
  reasoning.} With Qwen3-4B, every tested nonzero ratio reaches CTM
  $0.97$--$0.99$, compared with $0.76$ without replay. FinalAvg varies by only
  $1.3$ percentage points across the nonzero ratios. The best sequential value
  in each metric is bold.}
  \label{tab:cpr_ratio_robustness}
  \begin{tabular}{lrrrrr}
    \toprule
    Replay ratio $\rho$ & FinalAvg (\%) & CTM & FWT (\%) & TLG (\%) & BWT (\%) \\
    \midrule
    $0.00$ (Seq.\ RLVR) & 49.9 & 0.76 & +17.2 & +15.4 & \textbf{+1.6} \\
    $0.05$ & \textbf{64.5} & \textbf{0.99} & \textbf{+22.7} & +28.9 & $-2.6$ \\
    $0.20$ & 63.2 & 0.97 & +20.1 & \textbf{+29.2} & $-1.8$ \\
    $0.50$ & 63.3 & 0.97 & +20.7 & +26.0 & +1.2 \\
    \midrule
    MTRL reference & 65.3 & 1.00 & -- & -- & -- \\
    \bottomrule
  \end{tabular}
\end{table}

\subsection{CPR Gains Accompany Both Entropy Increases and Decreases}
\label{app:entropy-analysis}

We test whether lower final-policy entropy consistently accompanies CPR's
performance gains. For each method's final policy, we measure generated-token
predictive entropy on a fixed prompt set under greedy decoding. The reported
value is the token-weighted mean over generated response tokens, in nats.
Figure~\ref{fig:cpr-entropy} compares CPR with the Seq.\ RLVR baseline within each setting.

\begin{figure}[H]
  \centering
  \includegraphics[width=\linewidth]{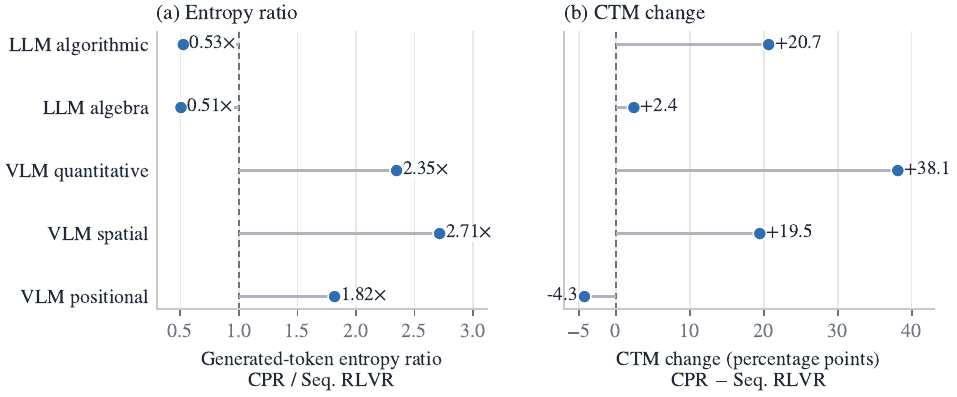}
  \caption{\textbf{CPR gains accompany both entropy increases and decreases.}
  Each row pairs CPR with Seq.\ RLVR in the same setting. The left panel
  reports the ratio of final-policy generated-token entropy, measured on fixed
  prompts under greedy decoding. The right panel reports the corresponding CTM
  change. Dashed lines mark no change. Among the four settings where CPR improves
  CTM, entropy decreases in both LLM settings and increases in VLM quantitative
  and spatial reasoning.}
  \label{fig:cpr-entropy}
\end{figure}

CPR gains occur under opposite entropy changes. Entropy falls by $47$--$49\%$
in the two LLM settings, where CTM rises by $20.7$ and $2.4$ percentage points.
In VLM quantitative and spatial reasoning, entropy rises to $2.35$ and $2.71$
times the corresponding baseline level while CTM rises by $38.1$ and $19.5$ points.
CPR's gains therefore do not consistently coincide with lower final-policy
entropy.

\clearpage
\subsection{Additional Evidence for Shared Reasoning}

We provide additional diagnostics for both roles of shared reasoning. We
compare the current-task direction with a mixture that includes previous-task
directions to test the mechanism CPR uses to improve FWT and TLG. We then examine
task-gradient alignment across domains. With the jugs trajectory, we connect
later-task support to modest forgetting and earlier-task support to forward
transfer at the behavioral level.

\subsubsection{Gradient Alignment Diagnostics}
\label{app:gradient-diagnostics}

We use two diagnostics. First, at the Qwen3-4B base model, we compare the
current-task gradient with a mixture that includes previous-task gradients.
Second, we compare task-gradient alignment across domains.

\paragraph{Adding previous-task gradients improves alignment with the all-task mean.}
At the Qwen3-4B base model $\theta^\star$, we evaluate all ten LLM
algorithmic task gradients and write
\begin{equation}
g_k^\star=\nabla J_k(\theta^\star),
\qquad
\bar g^\star=\frac{1}{10}\sum_{i=1}^{10}g_i^\star.
\end{equation}
Each $g_k^\star$ averages two mini-batch gradients from the language-model head and
final two transformer blocks. Combining all eight FSDP shards gives
$201{,}861{,}632$ coordinates. The parameters remain fixed throughout collection.

For each stage $t\in\{2,\ldots,10\}$ in the LLM algorithmic sequence, we compare
the current-task direction with an equal-weight previous-task mixture:
\begin{equation}
g_t^{\mathrm{current}}=g_t^\star,
\qquad
g_t^{\mathrm{mix}}=0.5g_t^\star+\frac{0.5}{t-1}\sum_{k<t}g_k^\star.
\end{equation}
Only current-task and previous-task gradients enter $g_t^{\mathrm{mix}}$.
We evaluate both directions by their alignment with the all-task mean
$\bar g^\star$. We use equal weighting to isolate the contribution of
previous-task directions from CPR's prompt selector.

Table~\ref{tab:mean-alignment-diagnostic} reports cosine and inner product with
$\bar g^\star$. Across stages 2 through 10, mean cosine rises from $0.43$ to
$0.79$. Mean inner product rises by $18.2\%$, while its standard deviation falls
by $55.3\%$. The mixture has the higher cosine in all nine comparisons and the
higher raw inner product in seven.

\begin{table}[H]
  \centering
  \small
  \setlength{\tabcolsep}{4.5pt}
  \caption{\textbf{Adding previous-task gradients improves alignment with the all-task mean at the Qwen3-4B base model.} Each row uses the same fixed set of ten task gradients. For each stage, the mixture assigns half of its weight to the current task and distributes the other half equally across previous tasks. The all-task mean $\bar g^\star$ serves only as the evaluation direction. Inner products are reported in units of $10^{-5}$.}
  \label{tab:mean-alignment-diagnostic}
  \begin{tabular}{clrrrr}
    \toprule
    & & \multicolumn{2}{c}{Cosine to $\bar g^\star$} & \multicolumn{2}{c}{Inner product $(10^{-5})$} \\
    \cmidrule(lr){3-4}\cmidrule(lr){5-6}
    Stage & Current task & Current & Mixture & Current & Mixture \\
    \midrule
    2 & Binary Alternation & 0.410 & 0.628 & 1.122 & 5.421 \\
    3 & Binary Matrix & 0.467 & 0.679 & 7.811 & 6.616 \\
    4 & Caesar Cipher & 0.351 & 0.804 & 1.384 & 3.801 \\
    5 & Cryptarithm & 0.372 & 0.811 & 0.964 & 2.987 \\
    6 & Isomorphic Strings & 0.751 & 0.849 & 17.877 & 11.039 \\
    7 & Jugs & 0.391 & 0.905 & 1.619 & 4.049 \\
    8 & Matrix Rotation & 0.360 & 0.761 & 3.394 & 4.590 \\
    9 & String Manipulation & 0.422 & 0.780 & 4.125 & 4.806 \\
    10 & A::B Rewriting & 0.344 & 0.929 & 1.051 & 3.193 \\
    \midrule
    Mean & & 0.430 & 0.794 & 4.372 & 5.167 \\
    Std. & & 0.120 & 0.092 & 5.217 & 2.330 \\
    \bottomrule
  \end{tabular}
\end{table}

\paragraph{Shared LLM geometry weakens across domains.}
Figure~\ref{fig:grad_cross_stream} combines separately collected CPR diagnostics
for the ten algorithmic tasks and six algebra tasks using the same model and
parameter subset. Across the 60 cross-domain task pairs, mean cosine is $-0.01$.
The cosine between the two domain-average gradient vectors is $0.006$. Both are
close to zero, placing the measured algorithmic--algebra geometry near
orthogonality. This panel characterizes cross-domain geometry. The comparison
above instead measures how adding previous-task gradients at the Qwen3-4B base model
changes alignment with the all-task mean.

\begin{figure}[H]
  \centering
  \includegraphics[width=0.92\linewidth]{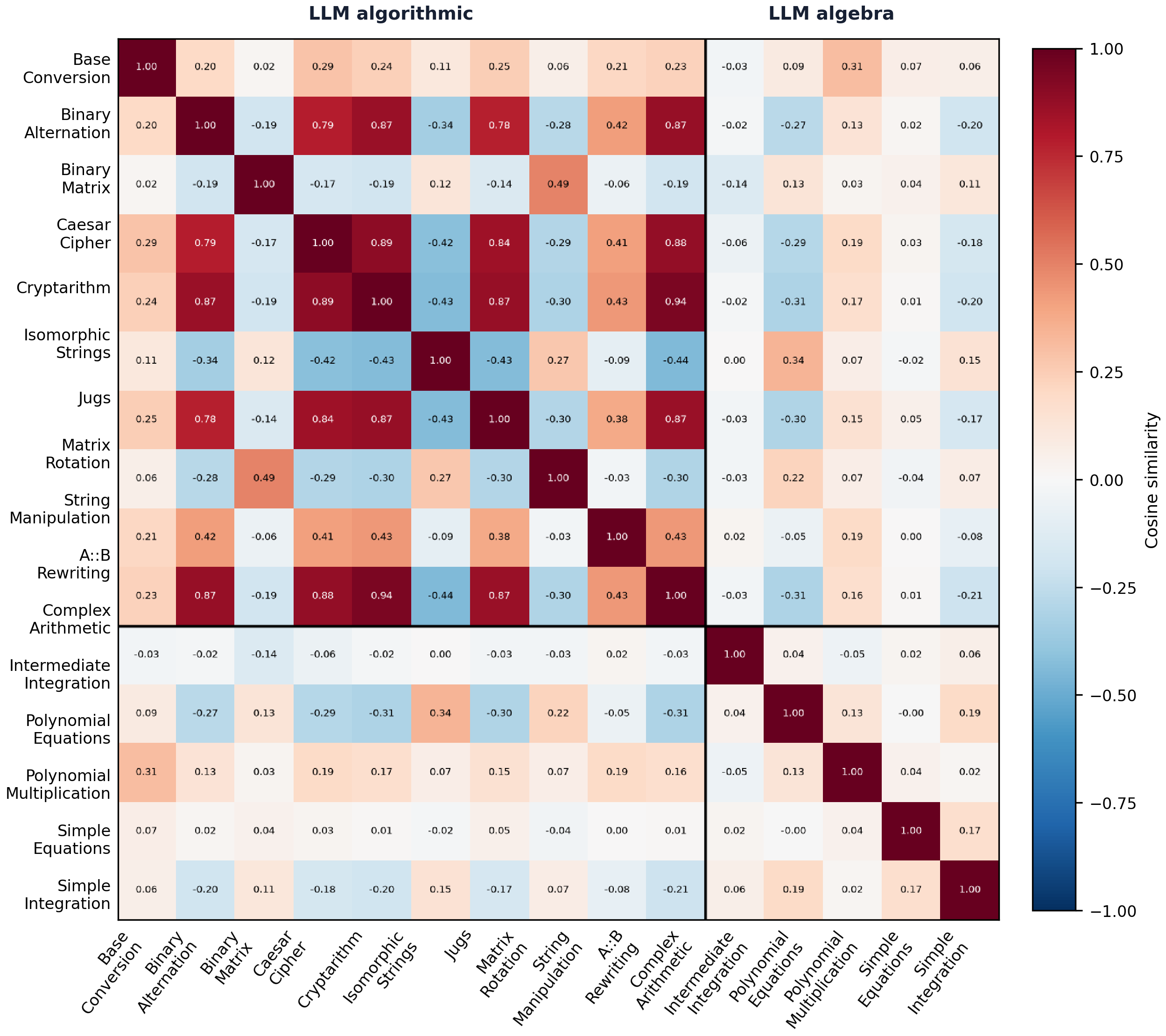}
  \caption{\textbf{LLM gradient alignment weakens across the
  algorithmic--algebra boundary.} Pairwise cosines are shown for the ten
  algorithmic tasks and six algebra tasks. Black lines mark the domain boundary.
  Cross-domain task pairs have mean cosine $-0.01$, while the cosine between the
  two domain-average vectors is $0.006$. The panel combines separately collected
  CPR diagnostics.}
  \label{fig:grad_cross_stream}
\end{figure}

\paragraph{Measured VLM gradients remain positively aligned across domains.}
Figure~\ref{fig:vlm_3stream} combines CPR diagnostics collected separately on
the quantitative, spatial, and positional domains. Every displayed off-diagonal
cosine is positive, including all cross-domain blocks. Comparing these positive
VLM cosines with the near-zero LLM geometry, we find that cross-task alignment
varies with the reasoning domain.

\begin{figure}[H]
  \centering
  \includegraphics[width=0.82\linewidth]{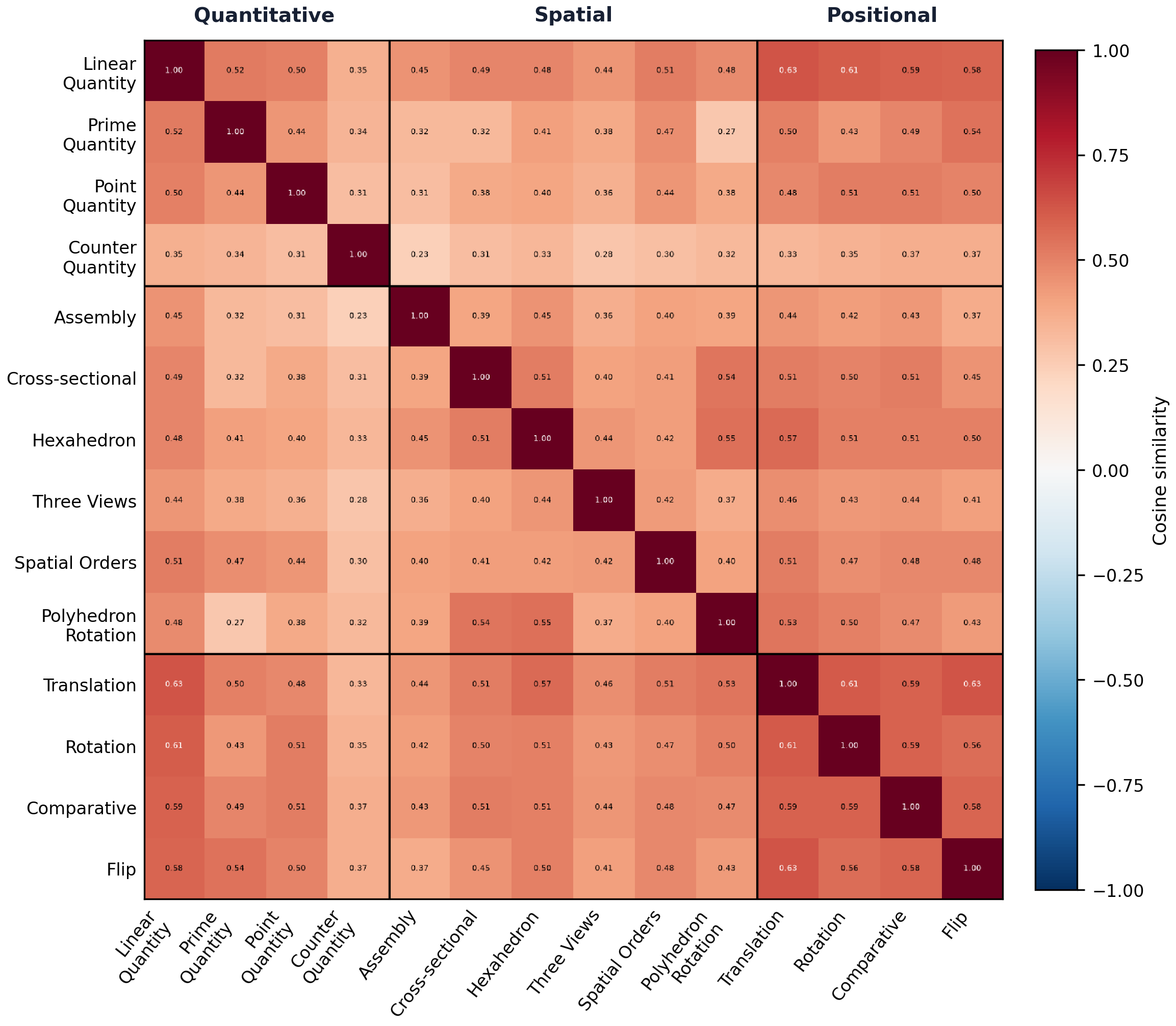}
  \caption{\textbf{The measured VLM task gradients are positively aligned
  across all three domains.} The matrix combines separately collected CPR
  diagnostics for quantitative, spatial, and positional reasoning. Every
  displayed off-diagonal task pair has positive cosine, including pairs that
  cross domain boundaries. Black lines mark those boundaries.}
  \label{fig:vlm_3stream}
\end{figure}

At the frozen LLM checkpoint, adding previous-task directions increases average
alignment to the all-task direction while reducing its variation across stages
2 through 10. Across domains, measured task-gradient alignment ranges from
near-orthogonal algorithmic--algebra geometry to uniformly positive VLM
geometry. We use this contrast to characterize the strength of shared reasoning
as domain-dependent.

\FloatBarrier

\subsubsection{Jugs Reasoning Trajectory}
\label{app:reasoning_chain_case_study}

We use the jugs trajectory to distinguish procedural reasoning from
verifier-valid answer generation across four checkpoints from the same Seq.\
RLVR run on the LLM algorithmic stream
(Figure~\ref{fig:reasoning_chain_case_study}). Before any RLVR, the base model explores alternative routes, derives the optimal
10-operation B-to-A route, and then continues to recheck it until generation
ends without an answer. The pre-jugs policy abandons an initial direction,
reaches A=3 through a longer 28-operation route, and likewise emits no answer.
After direct jugs training, the model derives the optimal route, explicitly
checks all 10 states, and emits a valid answer. The final policy produces the
same answer after later-task training while reducing the reasoning trace from
$1{,}631$ to $780$ characters and omitting the separate rechecking pass.

\paragraph{Raw transcripts.}
\label{app:jugs_full_raw_transcripts}

The four outputs summarized in Figure~\ref{fig:reasoning_chain_case_study}
appear below verbatim.

\casetitle{Base model / before RLVR}{jugs, LLM algorithmic}
\begin{lstlisting}[style=caseTranscript]
<think>
Okay, let's see. I need to figure out how to get exactly 3 liters in any of the jugs using the given moves. The jugs are A (10L), B (11L), and C (11L). The target is 3 liters. All start empty.

First, I remember that these types of puzzles usually involve filling, emptying, and pouring between jugs to reach the desired amount. Since the target is 3, which is smaller than the capacities of all jugs, maybe I need to find a way to measure it through some combination of pours.

Let me think about possible steps. Let's start by trying to fill the largest jugs first. Maybe fill B or C first. But since they are 11L, which is more than 3, maybe that's not directly helpful. Wait, but if I can pour from a full jug into another, maybe I can get the right amount.

Wait, the target is 3. Let me think of possible combinations. For example, if I can get 11 - 8 = 3, but how to get 8? Alternatively, maybe using the 10L jug. If I fill A (10L) and pour into B (11L), but B is empty. So pour A into B. Then B has 10L, A is empty. Then fill A again, pour into B. B can take 1L more (since it's 11L). So after pouring, B is full (11L), and A has 10 - 1 = 9L left. Then empty B. Now, pour the 9L from A into B. Now B has 9L, A is empty. Then fill A again (10L), pour into B. B can take 2 more liters (11-9=2). So after pouring, B is 11L, A has 10-2=8L. Then empty B again. Now pour A's 8L into B. B has 8L. Fill A again (10L), pour into B. B can take 3 more liters (11-8=3). So after pouring, A has 10-3=7L, and B is full. But that's not helpful. Wait, but maybe I can use another jug. Alternatively, maybe using C instead of B.

Alternatively, maybe using the 10L and 11L jugs. Let me think. If I fill B (11L) and pour into A (10L). Then B has 1L left. Then empty A. Now pour the 1L from B into A. Now A has 1L. Then fill B again (11L), pour into A until A is full. A can take 9 more liters. So B would have 11 - 9 = 2L left. Then empty A again. Pour the 2L from B into A. Now A has 2L. Fill B again (11L), pour into A. A can take 8 more liters, so B would have 11 - 8 = 3L. There we go! So B would have 3L. That's the target. Let me check the steps again.

So steps would be:

1. Fill B (11L) → B=11, A=0, C=0.
2. Pour B into A until A is full. A can take 10L, so B has 11-10=1L left. Now A=10, B=1, C=0.
3. Empty A → A=0, B=1, C=0.
4. Pour B into A. Now A=1, B=0, C=0.
5. Fill B again → B=11, A=1, C=0.
6. Pour B into A until A is full. A can take 9 more liters. So B becomes 11-9=2L. Now A=10, B=2, C=0.
7. Empty A → A=0, B=2, C=0.
8. Pour B into A. Now A=2, B=0, C=0.
9. Fill B again → B=11, A=2, C=0.
10. Pour B into A until A is full. A can take 8 liters, so B becomes 11-8=3L. Now B=3, which is the target.

So the moves would be:

fill B, pour B->A, empty A, pour B->A, fill B, pour B->A, empty A, pour B->A, fill B, pour B->A.

But let me
\end{lstlisting}

\medskip
\casetitle{Before jugs training / step 300}{jugs, LLM algorithmic}
\begin{lstlisting}[style=caseTranscript]
<think>
Okay, let's see. I need to figure out how to get exactly 3 units of water in any of the jugs using the given operations. The jugs are A (10L), B (11L), and C (11L). The target is 3L. All start empty.

First, I remember that these types of puzzles usually involve filling and transferring water between jugs to reach the desired amount. Since the target is 3, which is less than the capacities of all jugs, maybe I can use the smaller jugs to measure it.

Let me think. The jugs are 10, 11, and 11. The target is 3. So maybe I can use the 11L jugs to measure. For example, if I fill one of the 11L jugs and pour into the 10L jug, that would leave 1L in the 11L jug. Then, if I pour that 1L into the other 11L jug, maybe I can do something else. Wait, but that's not getting me to 3.

Wait, maybe I need to use the 10L jug. Let me try a different approach. Let's think of possible steps:

1. Fill the 10L jug (A) to 10L.
2. Pour from A to B (11L). Since B is empty, it can take all 10L. Now A is empty, B has 10L.
3. Fill A again (10L).
4. Pour from A to B. B can take 1L more (since it has 10L, and B is 11L). So after pouring, B is full (11L), and A has 10 - 1 = 9L left.
5. Now empty B (11L) to 0L.
6. Pour the 9L from A into B. Now B has 9L, A is empty.
7. Fill A again (10L).
8. Pour from A to B. B can take 2L (since it has 9L, needs 2 more to reach 11). So after pouring, B is 11L, A has 10 - 2 = 8L.
9. Empty B again.
10. Pour the 8L from A into B. Now B has 8L, A is empty.
11. Fill A again (10L).
12. Pour from A to B. B can take 3L (since it has 8L, needs 3 more to reach 11). So after pouring, B is 11L, A has 10 - 3 = 7L.
13. Empty B again.
14. Pour the 7L from A into B. Now B has 7L, A is empty.
15. Fill A again (10L).
16. Pour from A to B. B can take 4L (7 + 4 = 11). So A has 10 - 4 = 6L.
17. Empty B again.
18. Pour the 6L from A into B. Now B has 6L, A is empty.
19. Fill A again (10L).
20. Pour from A to B. B can take 5L (6 + 5 = 11). A has 10 - 5 = 5L.
21. Empty B again.
22. Pour the 5L from A into B. Now B has 5L, A is empty.
23. Fill A again (10L).
24. Pour from A to B. B can take 6L (5 + 6 = 11). A has 10 - 6 = 4L.
25. Empty B again.
26. Pour the 4L from A into B. Now B has 4L, A is empty.
27. Fill A again (10L).
28. Pour from A to B. B can take 7L (4 + 7 = 11). A has 10 - 7 = 3L. 

Ah, there we go. Now A has 3L. So the steps would be:

Fill A, pour to B, fill A, pour to B, empty B, pour A to B, fill A, pour to B, empty B, pour A to B, fill A, pour to B, empty B, pour A to B, fill A, pour to B, empty B, pour A to B, fill A, pour to B, empty B, pour A to B, fill A, pour to B.
\end{lstlisting}

\medskip
\casetitle{After jugs training / step 350}{jugs, LLM algorithmic}
\begin{lstlisting}[style=caseTranscript]
<think>
Okay, let's see. I need to figure out how to get exactly 3 gallons in any of the jugs. The jugs are A (10), B (11), and C (11). The target is 3. So, the possible moves are fill, empty, and pour.

First, I should think about the possible combinations. Since the jugs are 10, 11, and 11, maybe I can use the 11-liter jugs to measure out the 3. Let me try to think of steps.

Let me start by filling the 11-liter jug (B) first. Then, pour it into the 10-liter jug (A). That would leave 1 liter in B. Then, empty A. Now, pour the 1 liter from B into A. Now A has 1 liter, B is empty. Then, fill B again. Now, pour B into A until A is full. A can take 9 more, so B would have 11 - 9 = 2 left. Then empty A again. Now pour the 2 from B into A. Now A has 2. Fill B again. Now pour B into A. A can take 8 more, so B would have 11 - 8 = 3. There we go. So the steps would be:

Fill B, pour B into A. Empty A, pour B into A. Fill B, pour B into A. Empty A, pour B into A. Then fill B again and pour into A. That would leave 3 in B. So the moves would be:

fill B
pour B->A
empty A
pour B->A
fill B
pour B->A
empty A
pour B->A
fill B
pour B->A

Wait, but that's a lot of steps. Let me check again. Let me track each step:

1. Fill B (11)
2. Pour B into A (A is 10, B is 1)
3. Empty A (A is 0, B is 1)
4. Pour B into A (A is 1, B is 0)
5. Fill B (11)
6. Pour B into A (A can take 9, so B has 2 left)
7. Empty A (A is 0, B is 2)
8. Pour B into A (A is 2, B is 0)
9. Fill B (11)
10. Pour B into A (A can take 8, so B has 3 left)

So after step 10, B has 3. So the moves would be the list of these steps. So the answer is the list of these moves.
</think>

<answer>["fill B", "pour B->A", "empty A", "pour B->A", "fill B", "pour B->A", "empty A", "pour B->A", "fill B", "pour B->A"]</answer>
\end{lstlisting}

\medskip
\casetitle{Final policy / step 500}{jugs, LLM algorithmic}
\begin{lstlisting}[style=caseTranscript]
<think>
Okay, let's see. I need to get exactly 3 gallons in any of the jugs. The jugs are A (10), B (11), and C (11). The moves are fill, empty, and pour. Let me think.

First, maybe start by filling the largest jug, like B or C. Let's try filling B. Then pour from B to A. Since A is 10, and B is 11, pouring B into A would fill A to 10, leaving 1 in B. Then empty A. Now, pour the 1 from B to A. Now A has 1. Then fill B again. Now pour B into A. A can take 9 more, so B would have 11-9=2 left. Now B has 2. Then empty A. Pour B's 2 into A. Now A has 2. Fill B again. Pour B into A. A can take 8 more, so B would have 11-8=3. There we go, B has 3. So the moves would be: fill B, pour B->A, empty A, pour B->A, fill B, pour B->A, empty A, pour B->A, fill B, pour B->A. That gives B as 3.
</think>

<answer>["fill B", "pour B->A", "empty A", "pour B->A", "fill B", "pour B->A", "empty A", "pour B->A", "fill B", "pour B->A"]</answer>
\end{lstlisting}

Together, we use these diagnostics to distinguish the two roles of shared
reasoning. Positive mean gradient alignment persists under both Seq.\ RLVR and
CPR, while its magnitude varies across domains. In the jugs trajectory, the
learned solution persists through later training, and the final policy produces
it with a shorter reasoning trace, consistent
with modest forgetting. Success also improves before direct jugs training,
illustrating how shared reasoning can benefit forward transfer.

\subsection{Pass@\texorpdfstring{$K$}{K} Comparison with the Base Model}
\label{app:passk}

We examine whether CPR improves both single-sample accuracy and solution
coverage as the sampling budget increases. We compare Qwen3-4B before RLVR with
the final 500-step CPR policy with current-task replay. Each policy produces 16
responses at temperature $0.7$ for the same 64 held-out prompts per LLM
algorithmic task. We estimate pass@$1$, pass@$4$, pass@$8$, and pass@$16$ from
these responses following \citet{codex2021}.

\begin{table}[H]
  \centering
  \scriptsize
  \setlength{\tabcolsep}{5pt}
  \caption{\textbf{CPR improves both pass@$1$ and higher-$K$ success rates.}
  CPR with current-task replay outperforms the base model throughout the
  sampling curve. Both policies use the same 64 held-out prompts per task and
  sample 16 responses at temperature $0.7$. The higher value within each task
  and $K$ is bold.}
  \label{tab:passk}
  \begin{tabular}{llrrrr}
    \toprule
    Task & Policy & pass@1 & pass@4 & pass@8 & pass@16 \\
    \midrule
    Base Conversion & Base & 57.7 & 81.0 & 87.5 & 90.6 \\
                    & CPR + current & \textbf{98.1} & \textbf{99.8} & \textbf{100.0} & \textbf{100.0} \\
    Binary Alternation & Base & 0.6 & 1.8 & 2.7 & 3.1 \\
                       & CPR + current & \textbf{43.1} & \textbf{55.4} & \textbf{60.1} & \textbf{64.1} \\
    Binary Matrix & Base & 17.9 & 27.7 & 31.7 & 34.4 \\
                  & CPR + current & \textbf{51.0} & \textbf{56.6} & \textbf{58.9} & \textbf{60.9} \\
    Caesar Cipher & Base & 0.2 & 0.8 & 1.6 & 3.1 \\
                  & CPR + current & \textbf{5.8} & \textbf{9.5} & \textbf{10.2} & \textbf{10.9} \\
    Cryptarithm & Base & 0.0 & 0.0 & 0.0 & 0.0 \\
                & CPR + current & 0.0 & 0.0 & 0.0 & 0.0 \\
    Isomorphic Strings & Base & 59.4 & 88.7 & 95.0 & 98.4 \\
                       & CPR + current & \textbf{99.9} & \textbf{100.0} & \textbf{100.0} & \textbf{100.0} \\
    Jugs & Base & 0.0 & 0.0 & 0.0 & 0.0 \\
         & CPR + current & \textbf{51.5} & \textbf{53.1} & \textbf{53.1} & \textbf{53.1} \\
    Matrix Rotation & Base & 33.2 & 38.9 & 42.1 & 45.3 \\
                    & CPR + current & \textbf{97.6} & \textbf{99.6} & \textbf{99.9} & \textbf{100.0} \\
    String Manipulation & Base & 4.9 & 11.4 & 15.1 & 18.8 \\
                        & CPR + current & \textbf{72.9} & \textbf{84.2} & \textbf{87.9} & \textbf{90.6} \\
    A::B Rewriting & Base & 0.0 & 0.0 & 0.0 & 0.0 \\
                   & CPR + current & \textbf{36.0} & \textbf{55.1} & \textbf{62.7} & \textbf{68.8} \\
    \midrule
    Macro-average & Base & 17.4 & 25.0 & 27.6 & 29.4 \\
                  & CPR + current & \textbf{55.6} & \textbf{61.3} & \textbf{63.3} & \textbf{64.8} \\
    \bottomrule
  \end{tabular}
\end{table}

CPR raises macro pass@$1$ from $17.4\%$ to $55.6\%$ and macro pass@$16$ from
$29.4\%$ to $64.8\%$. Its gains remain large throughout the sampling curve:
$38.2$, $36.3$, $35.7$, and $35.5$ points at $K\in\{1,4,8,16\}$. At $K=16$,
CPR improves nine of the ten tasks, including jugs from $0.0\%$ to $53.1\%$
and A::B rewriting from $0.0\%$ to $68.8\%$.

\subsection{Trained Small Models versus Larger Zero-Shot Models}
\label{sec:frontier_reference}

How does training a small model compare with relying on a larger model's
zero-shot performance? We compare the trained Qwen3-4B and Qwen2.5-VL-7B
policies with Claude Opus~4.8, GPT-5.5, and larger same-family Qwen backbones.
The Qwen references are Qwen3-32B for text reasoning and
Qwen2.5-VL-32B-Instruct and Qwen2.5-VL-72B-Instruct for visual reasoning. Every
model receives the same prompt, verifier, fixed evaluation examples, and
1024-token response budget.

\begin{table}[h]
  \centering
  \small
  \setlength{\tabcolsep}{3.5pt}
  \caption{\textbf{CPR outperforms larger same-family zero-shot models across
  all five settings.} Scores are mean verifier accuracy on CRG's fixed
  evaluation sets with a 1024-token response budget. Claude Opus~4.8, GPT-5.5,
  and the larger Qwen backbones are evaluated zero-shot. Base is the untrained
  small model. The Seq.\ RLVR, CPR, and MTRL columns report trained policies.
  The best score in each row is bold.}
  \label{tab:frontier_reference}
  \begin{tabular}{lrrrrrrrr}
    \toprule
    Setting & Opus 4.8 & GPT-5.5 & Qwen-32B & Qwen-72B & Base & Seq.\ RLVR & CPR & MTRL \\
    \midrule
    LLM algorithmic & 53.7 & \textbf{77.7} & 29.3 & -- & 17.5 & 49.9 & 63.3 & 65.3 \\
    LLM algebra & 84.0 & \textbf{94.8} & 66.6 & -- & 43.2 & 88.2 & 90.4 & 90.0 \\
    \midrule
    VLM quantitative & 10.5 & 5.5 & 27.7 & 26.6 & 24.8 & 25.5 & \textbf{37.2} & 30.9 \\
    VLM spatial & 19.9 & 0.8 & 27.7 & 28.9 & 24.1 & 26.3 & \textbf{32.5} & 31.6 \\
    VLM positional & 19.1 & 8.6 & 27.7 & 25.8 & 28.7 & 34.7 & 33.2 & \textbf{35.3} \\
    \bottomrule
  \end{tabular}
\end{table}

CPR exceeds the larger same-family Qwen references in every setting. The trained
4B text policy outperforms Qwen3-32B, and the trained 7B visual policy outperforms
both Qwen2.5-VL-32B-Instruct and Qwen2.5-VL-72B-Instruct. The frontier models
show a modality-dependent pattern. GPT-5.5 remains strongest in both text
settings, while CPR exceeds Claude Opus~4.8 and GPT-5.5 in all three visual
settings.

\clearpage
\section{Extended Related Work}
\label{app:extended_related_work}

\paragraph{RLVR task suites.}
RLVR task suites supply verifier-scored reasoning problems. Reasoning Gym
\citep{reasoninggym2025} supplies 100+ procedural generators with algorithmic
verifiers and parametric difficulty. VisuLogic \citep{visulogic2025} is a
human-verified visual-reasoning benchmark with a rule-based RL baseline.
AgentGym-RL \citep{agentgymrl2025} trains multi-turn LLM agents from environment
outcome rewards. Verifiable rewards can also incentivize correct reasoning in
base LLMs \citep{rlvrreasoning2025}. These works establish task-generation and
reward interfaces for RLVR. Their training protocols expose a fixed task
inventory, with some curricula varying difficulty or interaction horizon.
CRG orders task identities into stages and evaluates one
policy after each arrival.

\paragraph{Forgetting in CL.}
Catastrophic forgetting describes the loss of earlier-task performance after
training on later tasks and remains a central failure mode in CL
\citep{ewc2017}. Continual World measures this behavior in reinforcement-learning
task sequences \citep{continualworld2021}, while TRACE treats earlier-task
performance as a primary endpoint for sequentially fine-tuned language models
\citep{trace2023}.
Reinforcement and supervised updates exhibit different forgetting behavior.
RL's Razor finds that online reinforcement learning forgets less than supervised
fine-tuning when the two are matched on new-task performance
\citep{rlsrazor2025}. Zhang et al. report that reinforcement fine-tuning causes less forgetting of
prior knowledge than SFT in multimodal language models
\citep{rftpreserve2025}. We take these results as evidence that
reinforcement-learning updates may cause less forgetting of earlier capabilities
than supervised updates. Low forgetting alone does not determine whether
sequential training reaches joint training on the final task set. Final
performance also depends on how well arriving tasks are learned and how training transfers across
tasks. This distinction motivates our comparison between continual RLVR and
MTRL.

\paragraph{CL methods.}
These methods alter how a learner updates across task stages.
Their mechanisms include regularization, plasticity restoration, parameter
isolation, replay, and optimization. EWC \citep{ewc2017} regularizes parameters
by Fisher information, ReDo \citep{redo2023} reactivates dormant units, FIRE
\citep{fire2026} controls re-initialization geometry, SPHERE \citep{sphere2026}
targets spectral plasticity, OSFT \citep{osft2025} constrains updates to a
low-rank subspace, and Muon \citep{muon2025} changes the optimizer geometry.
Their original evaluations span supervised CL, deep RL, and
large-model optimization. Our comparison places representatives of these
mechanisms under one continual-RLVR protocol and measures whether they narrow
the endpoint gap to MTRL.

\end{document}